\documentclass[11pt]{amsart}

\usepackage[utf8]{inputenc}
\usepackage{color}
\pdfoutput=1
\usepackage{geometry} 
\usepackage{graphicx,url}
\usepackage{amssymb,amsmath, amsthm}
\usepackage{epstopdf}
\usepackage{enumerate}
\usepackage{fancyhdr}
\usepackage{multirow}
\usepackage{tabularx}
\usepackage{subcaption}
\usepackage{longtable}
\usepackage{overpic}
\usepackage{algorithm, algorithmicx}
\newcommand{\re}{\mathbb{R}}

\newcommand{\cor}{\mathcal}

\usepackage{hyperref}
\usepackage{graphicx,url}

\newtheorem{theorem}{Theorem}
\newtheorem{proposition}{Proposition}
\newtheorem{lemma}{Lemma}
\newtheorem{remark}{Remark}[section]

\newtheorem{definition}{Definition}

\title{Pathwise optimization for bridge-type estimators and its applications}
\numberwithin{equation}{section}

\newcommand{\sgn}{\text{sgn}}
\newcommand{\de}{\mathrm{d}}

\usepackage{xcolor}

\allowdisplaybreaks
\date{\today}

\author{Alessandro De Gregorio}
\address{Department of Statistical Sciences, ``Sapienza" University of Rome,
	P.le Aldo Moro, 5 - 00185, Rome, Italy}
\email{alessandro.degregorio@uniroma1.it}

\author{Francesco Iafrate}
\address{Department of Mathematics, University of Hamburg,
Bundestr. 55, 20146 Hamburg, Germany}
\email{francesco.iafrate@uni-hamburg.de}
\begin{document}
	
\maketitle

\begin{abstract} 
Sparse parametric models are of great interest in statistical learning and are often analyzed by means of regularized estimators.
Pathwise methods allow to efficiently compute the full solution path for penalized estimators, 
for any possible value of the penalization parameter $\lambda$.  In this paper we deal with the pathwise optimization for  bridge-type problems; i.e. we are interested in the minimization of a loss function, such as negative log-likelihood or residual sum of squares, plus the sum of $\ell^q$ norms with $q\in(0,1]$ involving adpative coefficients. For some loss functions this regularization achieves asymptotically the oracle properties (such as the selection consistency). Nevertheless, since the objective function involves nonconvex and nondifferentiable terms, the minimization problem is computationally challenging.

The aim of this paper is to apply some   general algorithms, arising from nonconvex optimization theory, to compute efficiently the path solutions for the adaptive bridge estimator with multiple penalties. In particular, we take into account two different approaches: accelerated proximal gradient descent and blockwise alternating optimization.  The convergence and the path consistency of these algorithms  are discussed. In order to assess our methods, we apply these algorithms  to the penalized estimation of diffusion processes observed at discrete times. This latter represents a recent  research topic in the field of statistics for time-dependent data. 

\end{abstract}
	
{\it Keywords}:  adaptive thresholding operators, gradient descent, nonconvex optimization, path consistency, proximal maps, regularized estimators, stochastic differential equations

\section{Introduction}

It is well known that the optimization methods play a crucial role in machine learning, statistics and signal/image processing. Many practical problems  arising in statistical learning fall into the following class of problems
\begin{equation}\label{eq:ncproblem}
\min_{\theta} \left\{f(\theta):=   g(\theta) + \sum_{i=1}^m h_i(\theta)\right\},
\end{equation}
where $g:\mathbb R^{p}\to \mathbb R$ is a loss function, while the functions $h_i:\mathbb R^{p_i}\to \mathbb R,$ with $\sum_{i=1}^n p_i=p,$ represent nonsmooth and possibly nonconvex regularization terms.

 Sparse parametric models are of great interest in statistical learning and are often analyzed by means of regularized estimators.  A classical example is the linear regression model $y= {\bf X} \theta + \varepsilon,$ where $y$ is the response vector and ${\bf X}$ is the $n\times p$ predictor matrix, $\theta:=(\theta_1,...,\theta_p)^\top\in\mathbb R^p$ is the parametric vector and $\varepsilon$ is a centered sub-Gaussian random vector with independent components. In the high dimensional case, when $p > n$, the parameters can be estimated by means of regularized estimator, such as the LASSO, and many statistical guarantees can be obtained under the sparsity assumption, i.e. $\|\theta\|_0 = s < p$. 
 A more general regularized least squares problem is obtained 
 by solving the \eqref{eq:ncproblem} 
 with $g(\theta)=||y-{\bf X}\theta||^2$ and involving a $\ell^q$-penalty
$h(\theta)=\lambda||\theta||_q^q$. 
This leads to the bridge estimator $\hat \theta(\text{bridge})$, which is obtained 
 (see \cite{frank1993statistical}); 
 i.e.
 \begin{equation}\label{eq:bridge}
\min_{\theta}\left\{||y-{\bf X}\theta||^2+\lambda||\theta||_q^q\right\},
\end{equation}
where $||\theta||_q:=\left(\sum_{j=1}^p |\theta_j|^q\right)^{1/q}, q\in(0,1].$
\footnote{Even though  $||\theta||_q:=\left(\sum_{j=1}^p |\theta_j|^q\right)^{1/q}, q\in(0,1),$
is not a norm, we may still use this terminology as it is customary in literature.}  
The latter performs simultaneously both parameter estimation and variable selection. The $\ell^q$-penalties exhibit singularities (they are not differentiable); this allows for sparse estimation. For $q=1,$ the estimator \eqref{eq:bridge} becomes the Least Absolute Shrinkage and Selection Operator (LASSO) introduced in \cite{tib} and the optimization problem is convex.   Since for $q\in(0,1),$   The $\ell^q$ regularization term represents a bridge between the $\ell^0$-norm $\sum_{j=1}^p 1_{\theta_j\neq 0}$ and $\ell^1$ penalty. The reader can consult \cite{fu2000bridge},\cite{huang2008bridge} and \cite{group-bridge} for a discussion about the asymptotic properties of the bridge-type estimator. 
 \eqref{eq:bridge} represents a non-convex and non-smooth optimization problem and as such, despite its interesting theoretical properties, efficient computation of the path solutions for \eqref{eq:bridge} is a non-trivial matter.

Other types of nonconvex penalties have been dealt with in \cite{fanli2001} (Smoothly Clipped Absolute Deviation),  \cite{zhang2010} (Minimax Concave Penalty) and \cite{candes2008enhancing} (Log-Sum). In literature, some algorithms and methods have been proposed for nonconvex optimization problems such as local quadratic approximation (LQA) ( \cite{fanli2001}), minorization-maximization (MM) (\cite{hunter2005variable}), local linear approximation (LLA) (\cite{bridge-one-step}), and coordinate descent (\cite{breheny2011coordinate} and \cite{mazumder2011sparsenet}).

Although the non-convexity of the penalties leads to  overall non-convex optimization problems (and then computationally hard), numerous empirical studies have shown that gradient-based optimization methods, while only guaranteed to find local optima, often produce estimators having the oracle properties with consistently smaller estimation error than the LASSO estimators  (see, e.g., \cite{fanli2001} and \cite{bridge-one-step}).



A recent interesting application of the penalized estimation \eqref{eq:ncproblem} concerns the model selection for sparse stochastic processes; see, for instance, \cite{de2012adaptive}, \cite{masuda2017moment}, \cite{kinoshita2019penalized}, \cite{ciolek2020dantzig} and  \cite{ciolek2022lasso}, \cite{koike2020graphical}.
Therefore in this setting the regularization methods involve a loss function given, for instance, by the negative log-likelihood $\mathfrak L_n(\theta)$ (which  is generally not convex). 

\textbf{Mixed-rates asymptotics.}
There are several econometric models where the exact evaluation of the structural parameters requires estimators that exhibit \emph{mixed-rates asymptotics} (in \cite{antoine2012efficient}  some examples are discussed). 
This highlights the need for sparse parametric models  where different rates of convergence must be considered simultaneously. Then the problem becomes
\begin{equation}\label{eq:ncproblemsp}
\min_{\theta} \left\{\mathfrak L_n(\theta) +  \sum_{i=1}^m h_i(\theta)	\right\}.
\end{equation}
where each group of parameters has a different regularization term. 

A motivating example is given by discretely observed diffusion processes  $(X_t)_{t\geq 0}$ solution to the stochastic differential equation $$\de X_t =b(X_t,\alpha)\de t+ \sigma(X_t,\beta)\de W_t$$ with parametric vectors $\alpha\in\mathbb R^{p_1}$ and $\beta\in\mathbb R^{p_2}$ appearing in the drift and diffusion term, respectively. A suitable estimator of $\alpha$ and $\beta$ involves two different rates of convergence (see, e.g., \cite{yoshida2011polynomial}) and then the penalty term is split by grouping the parameters related to $b$ and $\sigma $.
In this framework, bridge and LASSO-type estimators have been studied in \cite{de2012adaptive}, \cite{suz} and \cite{di-regularized}, where $h_1(\alpha)$ and $h_2(\beta)$ are weighted $\ell^{q_i}$ penalties with  $q_i\in(0,1], i=1,2.$

By resorting to the same approach introduced in \cite{di-regularized}, it is possible to generalize the case of diffusion process to an arbitrary number of penalization terms. 
In order  to take into account the multiple asymptotic behavior of the non-regularized estimator, we suggest to  penalize different sets of parameters with different $\ell^q$ norms. 
Therefore, we are interested in to the estimator obtained as minimizer of the objective function \eqref{eq:ncproblemsp}, where $h_i$ represents an $\ell^{q_i}$ penalty.   We highlight again that the penalties as well as $\mathfrak L_n$ are not convex in general. 
The aim of this paper is to compute efficiently the solutions for the estimation/selection problems of type \eqref{eq:ncproblemsp} representing a non-convex optimization problem with non-differentiable penalties. We observe that  \eqref{eq:bridge}  falls into the class \eqref{eq:ncproblemsp}.

The classical convex optimization tools are not useful to compute efficiently  the estimator minimizing \eqref{eq:ncproblemsp}. For this reason, in the last decades, there has been an increasing interest for the nonconvex optimization problems and some algorithms has been proposed in literature (see, e.g., \cite{bolte2014proximal} and \cite{li2015accelerated}). By resorting to the recent  advances in the nonconvex optimization theory, we propose some algorithms for computing the pathwise solution (as the tuning coefficient varies) for a bridge-type estimators with multiple penalties arising from problem \eqref{eq:ncproblemsp}. Up to our knowledge,  this is the first attempt in the statistical literature to compute efficiently the pathwise solution  of a problem involving bridge-type constraints, without using convex relaxation of the penalties. Essentially, the present paper represents the  follow-up of \cite{di-regularized}, where we focus on the computational and algorithmic issues of the same problem. 

The paper is organized as follows. in Section \ref{sec:prox} we recall some notions from nonconvex optimization theory which are needed to deal with bridge-type estimators. Further details are given in the Appendix. In Section \ref{sec:bridge} we introduce the adaptive bridge-type estimators and the general statistical setting. The pathwise optimization algorithms are introduced in Section \ref{sec:opt}, where two methods are proposed: the first one is an accelerated gradient descent algorithm, while the second approach is based on the blockwise proximal alternating minimization. Furthermore, the convergence properties of the updates produced by the two methods are investigated.   
Section \ref{path} allows an analysis of the path consistency for each parameters. This property is close to the oracle features of the bridge estimators. Finally, Section \ref{sec:appglm}-\ref{sec:appsde} contain the applications of our methods. In particular, the algorithms are applied to a generalized linear models based on an exponential family of distributions (Section  \ref{sec:appglm}) and to an ergodic diffusion process (Section  \ref{sec:appsde}). Some simulations are performed and the results compared for different estimators.

\section{Preliminaries on $\ell^q$ regularizers}\label{sec:prox}

Let us assume that the function $ f $ is not differentiable and not necessarily convex. If  $ f = g + h $, where $ g $ is differentiable and $ h $ is not, it is possible to apply a proximal map method by minimizing the quadratic approximation to $g$ and leave $h$ alone. Define a \emph{proximal map } (in set-valued sense) as
\begin{equation}\label{eq:prox-ini}
	\text{prox}_{s, h}(x):= \arg \min_{u} \left \{ \frac 1{2s} \|x - u\|^2 + h (u)\right \}
\end{equation}
given $x\in\mathbb R^p$ and $s > 0 $. We observe that for $s>0$
$$\text{prox}_{s, h}(x)=\text{prox}_{1,s h}(x)$$
and $\text{prox}_{ h}(x):=\text{prox}_{1, h}(x)$. Then a \emph{proximal gradient descent algorithm} seeks a minimizer by performing
successive updates as follows  (see for more details on this method \cite{ boyd_vandenberghe_2004} and \cite{bertsekas2015convex})
\begin{align}\label{eq:prox-update}
\theta^{t} &\in\arg\min_\theta\left\{g(\theta^{t-1})+\nabla g(\theta^{t-1})^\top (\theta-\theta^{t-1})+\frac{1}{2s_t}||\theta-\theta^{t-1}||^2 + h(\theta) \right\}\\
&=\arg\min_\theta\left\{\frac{1}{2s_t}||\theta-(\theta^{t-1}-s_t \nabla g(\theta^{t-1}))||^2 + h(\theta) \right\}\notag\\
&= \text{prox}_{s_t, h}( \theta^{t-1} - s_t \nabla g(\theta^{t-1})).\notag
\end{align}
A crucial point is the choice of the stepsize: a common choice is to use a backtracking rule. 


In order to exploit algorithms based on the gradient descent in our framework, it is crucial to introduce a suitable proximal map \eqref{eq:prox-ini} with $h(\theta) = \lambda \sum_{i=1}^p w_i |\theta_i|^q, q\in(0,1],\lambda >0.$
For this reason, we deal with  the proximity operator
\begin{equation}\label{eq:prox}
\text{prox}_{\lambda||\cdot||_q^q}(z)=\arg \min_{\theta} \left \{ \frac 1{2} \|z - \theta\|^2 +\lambda ||\theta||_q^q\right \}=(\text{prox}_{\lambda|\cdot|^q}(z_1),\cdots, \text{prox}_{\lambda|\cdot|^q}(z_p))
\end{equation}
where  $z=(z_1,...,z_p)^\top\in \mathbb R^p.$
Therefore, it is sufficient to solve the one-dimensional optmization problem
\begin{equation}\label{eq:prox-lq}
	T^q_{ \lambda}	(z) :=\text{prox}_{\lambda|\cdot|^q}= \arg \min_{\theta \in \mathbb R} \left\{ \frac 12 (z - \theta )^2 + \lambda  |\theta|^{q} \right\},\quad z\in \mathbb R,
\end{equation}
where  $ T^q_{\lambda}$ denotes the adaptive thresholding map associated to the $ \ell^q $ metric.
It is possible to show that for $0<q<1$ (see Theorem 1 in \cite{lq-thresholding} and \cite{lqoptim}) the operator  $  T^q_{\lambda} $ can be expressed in the following form

\begin{equation}\label{eq:lq-threshold}
T^q_{\lambda}(z) = 
\begin{cases}
	0 & |z| < t_{q, \lambda} \\
	\{0, \sgn(z) \theta_{q, \lambda} \} & |z| = t_{q, \lambda} \\
	\sgn (z) \theta^*(z) & |z| > t_{q, \lambda}
\end{cases}
\end{equation}
where 
\begin{equation}\label{eq:tq-thresh-values}
	\theta_{q, \lambda} = [2 \lambda(1-q)]^{\frac{1}{2-q}}\,, \quad t_{q, \lambda} = \theta_{q, \lambda} + \lambda q \theta_{q, \lambda}^{q-1}
\end{equation}
and for $|z| > t_{q, \lambda},$
$ \theta^*(z) $ is the solution to 
\begin{equation}\label{eq:lq-root}
\theta  + \lambda q \theta^{q-1} = |z| \, , \quad \theta \in (\theta_{q, \lambda}, |z|).
\end{equation}
We have that  
\[\lim_{q\to 0^+}T^q_\lambda(z)=:
T^0_\lambda (z) = \begin{cases}
z & |z| > \sqrt {2\lambda} \\
\{0,z\}& |z|=\sqrt {2\lambda}\\
0 & |z| < \sqrt {2\lambda}
\end{cases}
\]
and 
\begin{equation}\label{eq:soft}
\lim_{q\to 1^-}T^q_\lambda(z)=:S_\lambda(z)= \sgn(z) (|z| - \lambda )_+ =
\begin{cases}
 \sgn(z) (|z| - \lambda ) & |z| > \lambda \\
 0 & \text{otherwise}
\end{cases}
\end{equation}	
representing the hard and soft-thresholding operators, respectively, used in best-subset selection and LASSO regression, respectively (see, e.g., \cite{hastiebook2015} and \cite{lq-thresholding}).

Let us note that there are two solutions at $ |z| = t_{q, \lambda}  $, but for practical purposes we will set $ T^q_{\lambda}(\pm t_{q, \lambda}) =0.$
In general it is not possible to solve \eqref{eq:lq-root} analytically and a numerical scheme is to be adopted, which can be done rapidly and allows a solution $\theta^*$ (see equation (5) in \cite{lq-thresholding}).

Besides the case $ q=1 $ (which recovers the soft-thresholding operator \eqref{eq:soft}), explicit expression for  $  T^q_{\lambda} $ are available only in the cases $ q=\frac 12  $ and $ q=2/3.$
Nevertheless, $ q=1/2 $ is considered to be superior in some applications
(see \cite{l12-regularization}). In particular for $ q= 1/2  $ one has
\begin{equation}\label{eq:l12-threshold}
T^{\frac 12}_{\lambda}(z) = 
\begin{cases}
0 & |z| \leq \frac 32 \lambda^{\frac 23}  \\
\frac 23 z \left(1 + \cos\left(\frac{2\pi}{3} - \frac 23 \arccos \left(\frac{\lambda}{4} \left(\frac{|z|}{3}\right)^{- \frac 32}\right)\right)\right) & |z| > \frac 32 \lambda^{\frac 23} 
\end{cases}
\end{equation}
with a slight modification w.r.t. to Theorem 1 in \cite{l12-regularization} to account for a factor of $ 1/2 $ in the objective function (see also \cite{lqoptim}).
Examples of the operators $ T^q_\lambda $ are shown in \autoref{fig:tq}. In particular in \autoref{fig:tq-family},
the operators corresponding to several different values of $ q $ are depicted. Notice how $ T^q_\lambda(z) $ 
it is not a continuous function of $ z $: the operators jump to zero, but the width of the jump is smaller than
the hard-thresholding operator $T^0_\lambda.$

Notice also that the jump points of $ T^q_\lambda $ do not change monotonically with $ q $.
In \autoref{fig:t13-soft} a comparison between $ T^\frac 13_\lambda $ and $ S_\lambda $ is shown.
Notice how for large values of $|z|$, $ T^q_\lambda(z)$ approaches the bisector, i.e. $ T^q_\lambda(z) \approx z, |z| >> t_{q, \lambda} $ (for $ q=1/2 $ this can be checked by a direct computation in \eqref{eq:l12-threshold}). This is a desirable behavior because values of the input away from zero remain unchanged. This corresponds to the fact that the $ \ell^q $ penalty, being concave, applies smaller penalization to large values of the parameters. The soft-thresholding operator instead systematically shifts the input to zero no matter its magnitude.
In a sense we can say that operators $ T^q_\lambda, 0< q < 1,$ interpolate between hard-thresholding 
and soft-thresholding. From the point of view of pathwise estimation this leads to discontinuous paths.

\begin{figure}
	\begin{subfigure}{0.7\textwidth}
		\centering
		\includegraphics[width=0.9\textwidth]{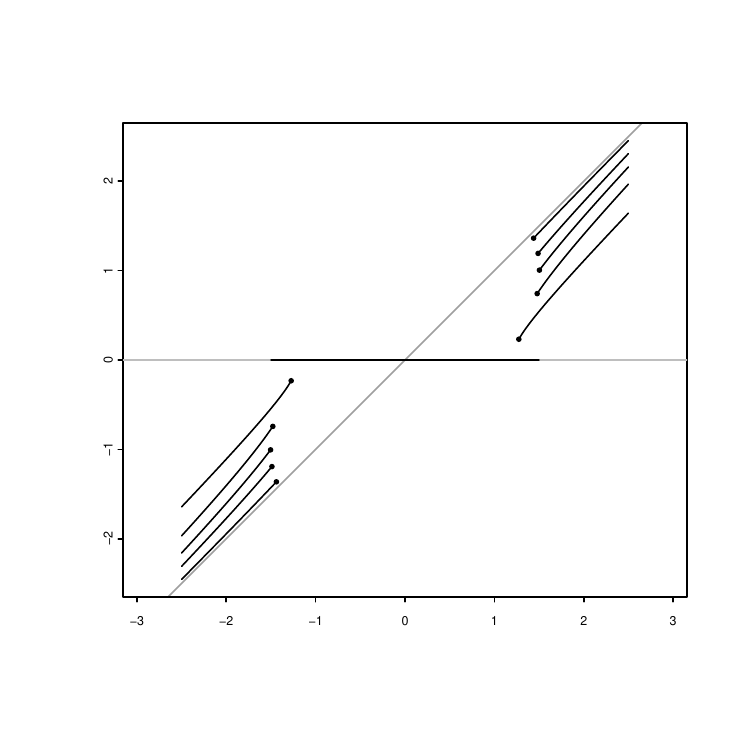}
		\caption{Family of $ T^q_1$ operators for $q=0.1, 1/3, 0.5, 2/3, 0.9$ from the ``furthest" to the $x$-axis to the ``closest" respectively. The bisector is shown in grey.}
		\label{fig:tq-family}
	\end{subfigure}
	\begin{subfigure}{0.7\textwidth}
		\centering
		\includegraphics[width=0.9\textwidth]{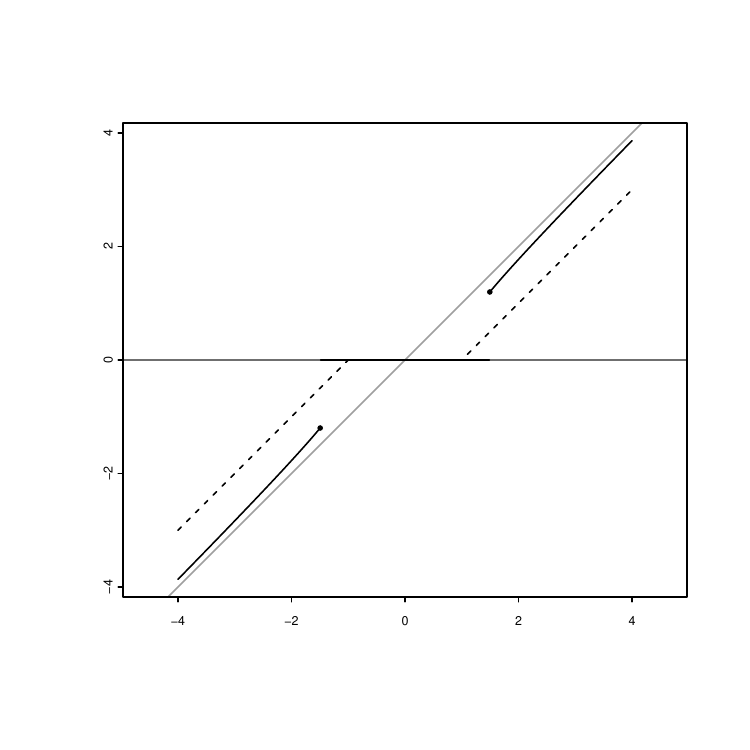}
		\caption{Comparison of $ T^{\frac 13}_1(z)$ (solid line) and the soft thresholding $ S_1(z)$ for larger values of $z$. The bisector is shown in grey.}
		\label{fig:t13-soft}
	\end{subfigure}%
	\caption{}
	\label{fig:tq}
\end{figure}

Let $ z = (z_1, \ldots, z_p)^\top \in \mathbb R^p,$  $ \mathbf w =  (w_1, \ldots, w_p)^\top \in \mathbb R_+^p $ and $ 0 <  q \leq  1.$ Let us denote with \begin{equation}\label{eq:hard-q}
\mathbf T_{\lambda \mathbf w}^{q}(z) := ( T_{\lambda  w_1}^{q}(z_1), \ldots, T_{\lambda  w_p}^{q}(z_p))^\top
\end{equation} 
the {\it adaptive $ q $-thresholding operator} which will play a crucial role in the optimization algorithms. For $q\to1^-,$ the operator \eqref{eq:hard-q} reduces to the component-wise adaptive soft thresholding operator as the vector map
$ \mathbf S_{\lambda\mathbf w} : \mathbb R^p \mapsto \mathbb R^p $ 
\begin{equation}\label{eq:soft-vec}
\mathbf S_{\lambda \mathbf w}(z) := (S_{\lambda w_1}(z_1), \ldots, S_{\lambda w_p}(z_p))^\top.
\end{equation}

Finally, by setting $h(\theta)=\lambda \sum_{i=1}^p w_i|\theta_i|^q, q\in (0,1],$ in the proximal map \eqref{eq:prox-update}, the above discussion allows to explicit the update as follows
\begin{equation}\label{eq:prox-st}
\theta^t={\bf T}_{\lambda s_t{\bf w}}^q(\theta^{t-1}-s_t \nabla g(\theta^{t-1})).
\end{equation}


\section{Bridge-type estimators with multiple penalties}\label{sec:bridge}

 Let us start recalling the shrinking estimators with multiple penalties introduced in \cite{di-regularized}. We deal with a parameter of interest $\theta:=(\theta^1,...,\theta^m)^\top,$ where $\theta^i:=(\theta_1^i,...,\theta_{p_i}^i)^\top, p_i\in\mathbb N,i=1,...,m.$ Furthermore, $\theta\in \Theta:=\Theta_1\times\cdots\times \Theta_m\subset \mathbb R^{\mathfrak{p}}, \mathfrak{p}:=\sum_{i=1}^mp_i,$ where $\Theta_i$ is a bounded convex subset of $\mathbb R^{p_i}.$ We denote by $\theta_0:=(\theta_0^1,...,\theta_0^m)^\top\in\text{int}(\Theta),$ where $\theta_0^i:=(\theta_{0,1}^i,...,\theta_{0,p_i}^i)^\top, i=1,...,m,$ stands for the true value of $\theta.$ Furthermore, we assume that $0\in \text{int}(\Theta).$

Let us assume that there exists  a loss function
$\theta\mapsto \mathfrak L_n(\theta)$ and $$\tilde\theta_n=(\tilde\theta_n^1,...,\tilde\theta_n^m)^\top\in\arg\min_\theta \mathfrak L_n(\theta).$$
Usually $\mathfrak L_n(\theta)$ is a (negative) log-likelihood function or the sum of squared residuals. Furthermore, suppose that $\tilde\theta_n$ admits a mixed-rates asymptotic behavior in the sense of \cite{rad}; that is for the asymptotic estimation of $\theta_0^i, i=1,...,m,$ is necessary to consider simultaneously different rates of convergence for $\tilde\theta_n^i, i=1,...,m$. More precisely, we deal with $r_n^i,i=1,...,m,$ representing sequences of positive numbers tending to 0 as $n\to\infty.$ ${\bf I}_m$ stands for the identity matrix of size $m$. Furthermore, we introduce the following matrices $$ A_n:= \text{diag}(r_n^1 {\bf I}_{p_1},...,r_n^m{\bf I}_{p_m}).$$
For the estimator $\tilde\theta_n$ the following assumptions hold true: 1) the estimator $\tilde \theta_n$ is consistent; i.e. 
\begin{equation}\label{eq:cond3}
A_n^{-1}(\tilde{\theta}_n-\theta_0)=\left(\frac{1}{r_n^{1}}(\tilde\theta_n^1- \theta_0^1),...,\frac{1}{r_n^{m}}(\tilde\theta_n^m- \theta_0^m)\right)^\top=O_p(1);
\end{equation}  $\tilde \theta_n$ is asymptotically normal; i.e. $$A_n^{-1}(\tilde{\theta}_n-\theta_0)=\left(\frac{1}{r_n^{1}}(\tilde\theta_n^1- \theta_0^1),...,\frac{1}{r_n^{m}}(\tilde\theta_n^m- \theta_0^m)\right)^\top\stackrel{d}{\longrightarrow} N_{\mathfrak p}(0,\mathfrak I),$$
		where $\mathfrak I:=\Gamma^{-1}$ and $\Gamma$ is a $\mathfrak{p}\times \mathfrak{p}$ positive definite symmetric matrix.

We assume that $\theta_0$ is sparse (i.e., some components of $\theta_0$ are exactly zero). Let $p_i^0:=|\{j:\theta_{0,j}^i\neq 0\}|, i=1,...,m,$ and $\mathfrak p^0:=\sum_{i=1}^mp_i^0.$ For the sake of simplicity, hereafter, we assume $\theta_{0,j}^i\neq 0, j=1,...,p_i^0,$ and it is equal to 0 otherwise, for any $ 
i = 1, ... , m.$ Therefore, our target is the sparse recovery of the model; i.e. we want to identify  the true model $\theta_0$ by exploiting a multidimensional random sample $(X_n)_{n}$ on some probability space. 
For this reason we use a penalized approach involving suitable shrinking terms. Since $\tilde \theta_n$ admits a mixed-rates structure, we penalize different sets of parameters with different norms. Therefore, the adaptive objective function with weighted $\ell^{q_i}$ penalties should be given by
\begin{align}\label{eq:stLASSO}
&\mathfrak L_n(\theta)+\left[\sum_{j=1}^{p_1}\lambda_{n,j}^1|\theta_j^1|^{q_1}+\sum_{j=1}^{p_2}\lambda_{n,j}^2|\theta_j^2|^{q_2}+...+\sum_{j=1}^{p_m}\lambda_{n,j}^m|\theta_j^m|^{q_m}\right]
\end{align}
where $q_i\in(0,1]$ and $(\lambda_{n,j}^i)_{n\geq 1}, j=1,...,p_i,i=1,...,m,$ are sequences of real positive random variable representing an adaptive amount of the shrinkage for  each element of $\theta^i.$

Hereafter,  the sequences of weights  $(\lambda_{n,j}^i)_{n\geq 1}$ will take the form
\begin{equation}\label{eq:lambda-n}
	\lambda_{n,j}^i = \lambda \, w_{n,j}^i \,, \,\, j=1,...,p_i,i=1,...,m
\end{equation}
where $ \lambda >0  $ is a constant not depending on $ n $  and $w_{n,j}^i>0$ (possibly random). Moreover for a vector $ z = (z_1, \ldots, z_p)^\top \in \mathbb R^p $ and  $ \mathbf w =  (w_1, \ldots, w_p)^\top \in \mathbb R_+^p,$  we denote the $ \mathbf w- $weighted $ \ell^q $ norm with
\begin{equation}\label{eq:w-lq-norm}
	\|\theta \|_{q,\mathbf w}^q := \sum_{j=1}^p w_j |\theta_j |^q,\quad  q>0.
\end{equation}

We follow two possible approaches. The first estimator is obtained by minimizing the cost function \eqref{eq:stLASSO}. In order to obtain a reasonable algorithm for the above problem, we must enforce some assumptions on $\mathfrak L_n.$

\begin{definition}  Let $\mathfrak L_n$ be a proper, coercive and $C^2$ function. We define the adaptive bridge-type estimator $\check{\theta}_n:\mathbb{R}^{(n+1)\times d}\to\overline \Theta$ as follows 
	
	\begin{equation}\label{eq:bridge-est2}
	\check{\theta}_n=(\check\theta_n^{1},...,\check\theta_n^{m})^\top\in\arg\min_{\theta\in\overline\Theta}\mathcal G_n(\theta; {\bf q})
	\end{equation}
where
\begin{equation}\label{eq: costfunction-f2}
\mathcal G_n(\theta; {\bf q}):=\mathfrak L_n(\theta) +  
\lambda \sum_{i=1}^m \|\theta^i\|_{q_i, \mathbf{w}_{n}^i}^{q_i}
\end{equation}
with $\mathbf{w}_{n}^i =  (w_{n,1}^i,w_{n,2}^i,...,w_{n,p_i}^i)^\top,$  $\mathbf q = (q_1, \ldots, q_m),$ and $q_i \in(0,1] \,, i=1,...,m.$ For $q_i=1,\,, i=1,...,m,$  \eqref{eq:bridge-est2} becomes a LASSO-type estimator.
\end{definition}
For $m=1,$ the estimator \eqref{eq:bridge-est2} has been studied in 
\cite{kinoshita2019penalized} with one penalty. The authors obtained the selection consistency and the convergence in law of \eqref{eq:bridge-est2} by means of a polynomial type large deviation inequality for  the statistical random field associated to $ \mathfrak L_n.$ We shed in light that the $\ell^{q_i}$-norms, with $q_i\in(0,1],$  allows non-differentiable  terms with some singularities. This choice is crucial to perform the selection of the true sub-model.

By setting 
\begin{equation}\label{eq:lsaf}
 \mathfrak L_n(\theta):=\frac 12 (\theta-\tilde{\theta}_n)^\top  G_n(\theta-\tilde{\theta}_n), 
\end{equation}
 where $ G_n$ be a $\mathfrak{p}\times \mathfrak{p}$   almost surely positive definite symmetric random matrix depending on $n$, $\check{\theta}_n$ becomes the least squares approximated estimator studied in \cite{wang1}, \cite{suz} and \cite{di-regularized}.  In this case, under suitable assumptions on $G_n$ and $\tilde\theta_n$, the regularized estimator \eqref{eq:bridge-est2} has the desirable oracle properties; that is $\check{\theta}_n$ is consistent, selects correctly the true sparse model and is asymptotically normal with reduced covariance matrix (see Theorem 1-3 in \cite{di-regularized}). For a discussion on the oracle features of a penalized estimator the reader can consult \cite{fanli2001}. 
 
 The loss function \eqref{eq:lsaf} is particularly useful in our framework, because allows to perform penalized estimation for diffusion processes. Therefore, hereafter $\hat\theta_n$ stands for the bridge estimator arising from problem \eqref{eq:bridge-est2} with loss function \eqref{eq:lsaf}.

Finally, we introduce the following notation.
\begin{itemize}
\item For the $ \mathfrak p \times \mathfrak p $ matrix $ G_n = (g_{k\ell}: k,\ell = 1, \ldots, \mathfrak p )$, with $ \mathfrak p = \sum_{h =1}^m p_h $, we write $ G_i $ to denote the $ i- $th block of rows of the matrix, i.e. the $ p_i \times \mathfrak p $ matrix $$ G_i :=\left (g_{k\ell}: k = \sum_{h =1}^{i-1} p_h +1 , \ldots, \sum_{h =1}^{i} p_h,\, \ell= 1\ldots, \mathfrak p \right),$$ for $ i = 1, \ldots, m $.

\item Let $ G_{ij} $ the $ j-$th row vector of the matrix $ G_i$, i.e. $$G_{ij} := \left(g_{k\ell}: k = \sum_{h =1}^{i-1} p_h +j, \,\ell = 1, \ldots \mathfrak p\right),$$ for $ i = 1, \ldots, m $, $ j = 1, \ldots, p_i $.

\item We use the notation $ \mathfrak G_i $ to denote the principal submatrix of $ G $ with columns and rows corresponding to the $ i-$th block, i.e. the  $ p_i \times p_i $ matrix
defined as $$\mathfrak G_i := \left(g_{k\ell}: k,\ell= \sum_{h=1}^{i-1} p_h +1 , \ldots, \sum_{h =1}^{i} p_h\right).$$

\item For a positive definite matrix $ G $ denote with $ \Lambda(G) $ its largest eigenvalue.

\item We denote with $ D_i $ the $ p_i \times p_i $ diagonal matrix with entries $g_{kk},$ i.e. $$D_i :=\text{diag}\left(g_{kk}:  k = \sum_{h =1}^{i-1} 
 p_h +1 , \ldots, \sum_{h =1}^{i} p_h \right). $$

\end{itemize}

\section{Pathwise optimization for bridge-type estimators}\label{sec:opt}

In practise, it is not easy to calculate the values of \eqref{eq:bridge-est2} for any value of $    \lambda$. This depends on the nonconvex and nonsmooth nature of our problem which falls into the class \eqref{eq:ncproblem}. By resorting to some tools arising from the nonconvex optimization theory, we introduce efficient algorithms for computing the adaptive bridge-type estimator \eqref{eq:bridge-est2}. In particular we focus our attention on the pathwise solution of the estimator. 

A linear approximation to the bridge penalty around an initial estimate $ \tilde{\theta} $ (see \cite{bridge-one-step}) has been considered in the literature. This approach regains convexity and results
	in a ``modified" soft-thresholding operator.  We suggest two procedures which don not require convex relaxation of the penalties; they are obtained by specializing the following two optimization algorithms:

\begin{itemize}
	\item Accelerated gradient descent approach for non-convex problems as in \cite{li2015accelerated}, which is a modification of the  algorithm introduced in \cite{5173518}.

	\item Block coordinate optimization for non-convex problems, see \cite{bolte2014proximal} (see also \cite{block-non-convex} or \cite{ipalm}). This respects the grouped structure of the original problem. 

\end{itemize}

	
	Hereafter, we assume that for any starting point $\theta^0\in\text{int}(\Theta)$ of the algorithm, the associated level set; i.e. 
$\{\theta\in \mathbb R^{\mathfrak p}: \mathcal{G}_n(\theta)\leq \mathcal{G}_n(\theta^0)\},$
is contained within $\text{int}(\Theta).$

\subsection{Accelerated proximal gradient algorithm}
A popular approach to penalized problems is to consider iterative proximal gradient algorithms (see Section \ref{sec:prox}). Accelerated algorithms rely on an interpolation between the current estimate and the previous steps. This ideas are exploited in the Fast Iterative Shrinkage-Thresholding Algorithm (FISTA) in \cite{nesterov1983method}, \cite{fista} and \cite{5173518} in the convex framework. FISTA has widely been applied for $ \ell^1 $ penalized estimation problems. Essentially, at each iteration the accelerated algorithm extrapolates a point by combining the current point and the previous point, then a proximal map function is iteratively applied till the convergence of the algorithm is reached.


Let us introduce the fast optimization method extending the FISTA algorithm to the nonconvex problem \eqref{eq:bridge-est2}. First, we observe that for the gradient descent approach with function $ h $  given by the weighted $ \ell^q  $ norm \eqref{eq:w-lq-norm}, the proximal map is equal to (see Section \ref{sec:prox})
\begin{equation}\label{eq:prox-l1}
\text{prox}_{s,\, \lambda||\cdot||_{q, \mathbf w}^q}(x)= \arg \min_{u} \left \{ \frac 1{2s} \|x - u\|^2 + \lambda \| u \|_{q, \mathbf w}^q \right \} = 
\mathbf{T}_{\lambda s \mathbf w}^{ q}(x)
\end{equation}
where $x\in\mathbb R^p,s>0, 0<q\leq 1$ and $\mathbf{T}_{\lambda \mathbf w}^{ q}$ is the component-wise adaptive hard-thresholding operator \eqref{eq:hard-q}.  
Let us introduce
\begin{equation}\label{eq:prox-q-s2}
\mathbf {\overline T}^{\mathbf q}_{\lambda s \mathbf w} (\theta) := \big (\mathbf T_{\lambda  s\mathbf w^i}^{q_i}( \theta^{i}
 -s\nabla_i \mathfrak L_n(\theta) ): i=1, \ldots, m \big ),\quad s>0, \theta\in\Theta,
\end{equation} 
where $ \nabla_i g$ stands for the ``partial" gradients obtained deriving $ g $ w.r.t. the components of $ \theta^i, i=1,2,...,m.$ For $q_i=1,i=1,2,...,m,$ the proximal map \eqref{eq:prox-q-s2} reduce to the soft-thresholding operator
\begin{equation}\label{eq:prox-q-ss}
\mathbf {\overline S}_{\lambda s \mathbf w} (\theta) := 
\big (\mathbf S_{\lambda s  \mathbf w^i}( \theta^i  -s\nabla_i \mathfrak L_n(\theta) ): i=1, \ldots, m \big ), \quad s>0,
\end{equation} 
where $\mathbf S_{\lambda s  \mathbf w^i}$ is given by \eqref{eq:soft-vec}.
  
For any fixed $ n \in \mathbb N$  and $ \lambda >0 $, we present an implementation of the \emph{monotone} Accelerated Proximal Gradient (APG) proposed in \cite{li2015accelerated} for problem \eqref{eq: costfunction-f2} involving nonconvex and nonsmooth multiple adaptive penalties. We denote by $\check \theta_n( \lambda)$ the bridge-type estimator. 
For the sake of simplicity, in what follows  we set $h:=h(\theta) := \lambda \sum_{i=1}^m \|\theta^i\|_{q_i,\mathbf w^i}^{q_i},$ $\check \theta^t := \check \theta^t( \lambda) $ and  drop the subscript $n.$ Furthermore, since $\mathfrak L_n$ is $C^2$ follows that it  has Lipschitz continuous gradient  with global Lipschitz constant $L(\mathfrak L_n)$ (see Appendix for the exact definition).

\begin{algorithm}[H]
	\caption{Pathwise accelerated proximal gradient descent algorithm }\label{alg:pathwise2bis}
	\begin{algorithmic}
		
		\State

		\begin{enumerate}[1.]
			\item Fix $ \lambda >0 $ and an initial values  $  {\eta}^{0}\in \text{int}(\Theta) $ (possibly depending on $ \lambda $).  Initialize $\check \theta^1=\check \theta^0 =  {\eta}^{0}, c_1=1$, $s<\frac1L, u<\frac1L,$ where $L$ is an upper bound of  $L (\mathfrak L):=L (\mathfrak L_n).$ 
			\item At step $ t= 1,2,...$, 
						 compute 
		\begin{equation}\label{eq:update-prox-eta2}
		\eta^{t} = \check \theta ^t +  \frac{c_{t-1}}{c_{t}}(\zeta^t - \check \theta^{t-1}) + \frac{c_{t-1}-1 }{c_{t}}(\check \theta^t - \check \theta^{t-1})
		\end{equation}
		\begin{equation}\label{eq:update-prox-zeta2}
		\zeta^{t+1} =\text{prox}_{s,h }( \eta^t - s \nabla \mathfrak L_n(\eta^t))= \mathbf {\overline T}^{\mathbf q}_{\lambda s \mathbf w} (\eta^t ),
		\end{equation}
		\begin{equation}\label{eq:update-prox-y2}
		\upsilon^{t+1} =\text{prox}_{u,h }( \check\theta^t - u \nabla  \mathfrak L_n(	\check\theta^t))= \mathbf {\overline T}^{\mathbf q}_{\lambda u\mathbf w} (\check \theta^t ),
		\end{equation}
		\begin{align}\label{eq:mom-weight-q2}
		c_{t+1} &= \frac{1 + \sqrt{1 + 4 c^2_t}}{2} \\
		\check \theta^{t+1} &=\begin{cases}
		\zeta^{t+1} & \text{if } \mathcal G(\zeta^{t+1}) \leq \mathcal G(\upsilon^{t+1})  \\
		\upsilon^{t+1} & \text{otherwise}
		\end{cases}
		\end{align}
					
			\item Repeat the previous steps over a grid of $ \lambda $ values to get the full path of the coefficient estimates.
		\end{enumerate}
	\end{algorithmic}
\end{algorithm}

From \eqref{eq:prox}, \eqref{eq:prox-lq}, \eqref{eq:prox-l1} and \eqref{eq:prox-q-s2}, it is not hard to verify \eqref{eq:update-prox-zeta2} and \eqref{eq:update-prox-y2} hold. Furthermore, Algorithm \ref{alg:pathwise2bis}   as shown in \cite{fista} and \cite{li2015accelerated} can achieve a $O(1/t^2) $ rate a of convergence for convex cost functions and the monotony in the nonconvex case, in contrast to the non-accelerated proximal gradient algorithm.  The crucial point is the monitor $\upsilon$ (see  \cite{li2015accelerated}), which allows to correct the accelerated term $\eta^{t+1}$ when it has the potential to fail; this  ``monitored acceleration" ensure that the accumulation point is a critical point and then it guarantees the convergence of the updates $\hat \theta^t, t=1,2,...$. We observe that the constant stepsizes appearing in Algorithm \ref{alg:pathwise2bis}, could be computed  by backtracking line search which allows to provide steps with different length at each iteration (see for more details \cite{li2015accelerated}).

It is an easy task to verify the assumptions required in Theorem 1  in \cite{li2015accelerated} for the convergence of  Algorithm \ref{alg:pathwise2bis}. Therefore, for each $ \lambda >0,  n\in\mathbb N $ and $ q_i \in (0,1], \, i=1,\ldots m $, the sequences of bounded estimates $ \{\check \theta^t_n: t = 1, 2, \ldots \} $ produced by  Algorithm \ref{alg:pathwise2bis} converge to a critical point $\check\theta^*_n$ of $\mathcal G_n(\theta; {\bf q});$ i.e. \begin{equation}\label{eq:convres2}
 \check\theta^*_n:=\lim_{t\to \infty}\check \theta^t_n, \quad \text{with}\, \,0\in \partial \mathcal G_n(\check\theta^*_n; {\bf q}).\end{equation}
 For other results on the convergence of the algorithm the reader can consult Theorem 3 in  \cite{li2015accelerated}.

 \begin{remark}
     If $\mathfrak L_n(\theta):=\frac 12 (\theta-\tilde{\theta}_n)^\top  G_n(\theta-\tilde{\theta}_n),$ the bridge-type estimator $\hat\theta_n$ reduces to the regularized estimator studied in \cite{suz} and \cite{di-regularized}. Clearly the proximity operator \eqref{eq:prox-q-s2} becomes $\mathbf {\overline T}^{\mathbf q}_{\lambda s \mathbf w} (\theta) := \big (\mathbf T_{\lambda  s\mathbf w^i}^{q_i}( \theta^{i}
 -sG_n(\theta-\tilde\theta_n)  ): i=1, \ldots, m \big ),\quad s>0, \theta\in\Theta.$ Furthermore, if $q_i=1,i=1,2,...,m,$ $\mathcal G_n$ becomes a convex cost function and the the LASSO-type estimator $\check\theta_n$ represents the global minimum of the problem. Furthermore, since $\nabla \mathfrak L_n(\theta)=G_n(\theta-\tilde{\theta}_n),$ the Lipschitz constant $L(\mathfrak L_n)$ becomes by the spectral norm $||G_n||=\Lambda(G_n)$ where $\Lambda(G_n)$ stands the largest eigenvalue of $G_n.$ Therefore, from Theorem 2 in   \cite{li2015accelerated}, we can conclude that for $q_i=1$ the  sequence $ \{\hat \theta^t_n, t = 1, 2, \ldots \} $ generated by  Algorithm \ref{alg:pathwise2bis} by applying the proximal map \eqref{eq:prox-q-ss}, approaches to the LASSO-type estimator $\hat \theta_n$ with an $O(1/t^2)$ convergence rate; i.e.
	\begin{equation}
	\mathcal G_n(\hat \theta_n^t)-\mathcal G_n(\hat \theta_n)\leq \frac{2\Lambda (G_n)||\hat \theta_n^0 -\hat \theta_n||^2}{\alpha(t+1)^2}, \quad \hat \theta_n^0\in\text{int}(\Theta).
	\end{equation}
 \end{remark}



\subsection{Blockwise proximal alternating minimization}

In order to take into account the multiple-penalties structure of our problem, we also consider a (non-accelerated) \emph{block} coordinate descent approach, i.e. we update one block of parameters at each step. The main point is to consider the approximations of our cost function by means of the standard proximal linearization (see \eqref{eq:prox-update}) of each subproblem and  alternating minimization on each block $\theta^i\in \Theta_i,i=1,...,m,$ by the proximal maps. This scheme is known in the literature as Proximal Alternating Linearized Minimization (PALM) and it has been introduce in \cite{bolte2014proximal}.

 For the sake of simplicity, in what follows  we set $\mathcal{G}(\theta):=\mathcal{G}_n(\theta;{\bf q}), \check \theta^t := \check \theta_n^t( \lambda) $ and  drop the subscript $n.$ Let us generate a sequence of estimates $\{\check \theta^t: t=1,2,...\}$ as follows. Let  $\mathfrak L$ be  a KL function (see Appendix for the definition) and $q_i\in(0,1]\cap \mathbb Q, i=1,2,...,m.$

\begin{algorithm}[H]
	\caption{Pathwise PALM algorithm }\label{alg:pathwise3}
		\begin{algorithmic}
		
		\State


		\begin{enumerate}[1.]
			\item Fix $ \lambda >0 $,  some initial values  $  {\check \theta^0(\lambda)}\in \text{int}(\Theta) $.
			
	\item[2.] At step $ t=1,2,...$, cycling  through the block components of $ \theta = (\theta^1, \ldots, \theta^m)^\top $
	(\emph{block proximal gradient}) update the $ i- $th block as  
	\begin{align}\label{eq:update-prox-i-grad}
	\check \theta^{i,t}&= \text{prox}_{\lambda s^i \|\cdot\|_{q_i, \mathbf{w}^i}^{q_i}}( \check\theta^{i,t-1}-s^i\nabla_i \mathfrak{L}(\check \theta^{1,t},...,\check \theta^{i-1,t},\check \theta^{i,t-1},...,\check \theta^{m,t-1})\\\notag
 &=\mathbf {T}^{q_i}_{\lambda s^i \mathbf w^i} \left(\check \theta^{i, t-1}  - s^i \nabla_i\mathfrak L\left(\check \theta^{1,t},...,\check \theta^{i-1,t},\check \theta^{i,t-1},...,\check \theta^{m,t-1} \right)\right)
	\end{align}
	with step size $ s^i = \alpha_i/L_i(\mathfrak L)$ with $0<\alpha_i<1$.
			
			\item[3.] Repeat the previous steps over a grid of $ \lambda $ values to get the full path of the coefficient estimates.
		\end{enumerate}
	\end{algorithmic}
\end{algorithm}	

As observed for Algorithm \ref{alg:pathwise2bis}, it is not hard to prove \eqref{eq:update-prox-i-grad}.
Note that in this case constant step sizes are used but, as noted in \cite{ipalm}, it is possible to implement backtracking rules also in the block coordinate case. 

 Algorithm \ref{alg:pathwise3} has two fine properties: 1)  for the updates  \eqref{eq:update-prox-i-grad}  the statement \eqref{eq:convres2} fulfills; i.e.  the bounded sequence $\{\check \theta^t: t=1,2,...\}$ is convergent to some accumulation/critical points; 2) if we chose a starting point quite close to the bridge-type estimate $\check\theta_n,$ then $\{\check\theta^t: t=1,2,...\}$ tends to $\check\theta_n$ (see Theorem 2.12 in \cite{Attouch2013ConvergenceOD} for the a detailed statement). 

The adaptive norms $\|\cdot\|_{q_i, \mathbf{w}^i}^{q_i}$ are semi-algebraic functions (see Appendix) and then $\mathcal G_n$ is a KL function. We observe that $\|\cdot\|_{q_i, \mathbf{w}^i}^{q_i}$ are bounded level sets and then the updates sequence $\{\check \theta^t: t=1,2,...\}$ is bounded. Since it is easy to check that $\mathcal G_n$ satisfies Assumption 1-2 in \cite{bolte2014proximal} (see also Remark 3 in the cited paper), the property 1) follows from Theorem 1 in \cite{bolte2014proximal}.    The property 2) is a consequence of Lemma 3-4 in \cite{bolte2014proximal}  and Theorem 2.12 in \cite{Attouch2013ConvergenceOD}.

\begin{remark}
Clearly, a fine example is given again by $\mathfrak L_n(\theta)=\frac 12 (\theta-\tilde{\theta}_n)^\top  G_n(\theta-\tilde{\theta}_n).$ In this case, the $i$-th block is updated with  \begin{align}\label{eq:udblsa}
	\hat \theta^{i,t} 
&=\mathbf {T}^{q_i}_{\lambda s^i \mathbf w^i} \left(\hat \theta^{i, t-1}  - s^i G_i\left((\hat \theta^{1,t},...,\hat \theta^{i-1,t},\hat \theta^{i,t-1},...,\hat \theta^{m,t-1} )^\top- \tilde \theta_n \right)\right).
	\end{align}
Moreover, $\frac 12 (\theta-\tilde{\theta}_n)^\top  G_n(\theta-\tilde{\theta}_n)$ is a  semi-algebraic function (see the discussion in Appendix) and then $\mathcal G_n$ is semi-algebraic as well. Therefore, the convergence results following Remark 6 in \cite{bolte2014proximal}, hold true.
\end{remark}

\section{Path consistency}\label{path} 

Define $ \lambda_{\max}$ as the smallest penalization value such that the null vector is a stationary point for \eqref{eq:stLASSO}. From result \eqref{eq:convres2}, we derive that a sufficient condition for this to happen is that the null-vector is a fixed point of the threshold updates $ \mathbf {\overline T}^{\mathbf q}_{\lambda s \mathbf w} $.
So we define 
\begin{equation}\label{eq:lambda-max-def}
    \lambda_{\max} :=\sup_s \inf\{\lambda >  0 : \mathbf {\overline T}^{\mathbf q}_{\lambda s \mathbf w}(0) = 0\},
\end{equation}
with $0\in\mathbb R^{\mathfrak p},$ where $ s $ is suitably bounded, e.g. $ s < 1/L(\mathfrak L_n) $.  

For the bridge-type estimator obtained by means of the least squares approximation, we replace $ \mathbf {\overline T}^{\mathbf q}_{\lambda s \mathbf w} $ with $ \mathbf {\widetilde T}^{\mathbf q}_{\lambda s \mathbf w}.$
	
\begin{remark}
	The sequence of $ \lambda $ values should start with the smallest $ \lambda_{\max} $ such that all the parameter as estimated 
	as zero and proceed backwards, using the previous estimate as starting point. This allows for greater efficiency in practice (see, e.g., Chapter 5, \cite{hastiebook2015}.)
\end{remark}

\begin{proposition}
	For Algorithm 1 , if $0< q_i <1, i=1,2,...,m$,
	\begin{equation}\label{eq:lambda-max-q}
		\lambda_{\max} =  \max_{\substack{1 \leq i \leq m \\ 1 \le j \leq p_i}} (w^i_{j})^{-1} (|\nabla_{ij}  \mathfrak L_n (0)|/c_{q_i})^{2-q_i}L(\mathfrak L_n)^{q_i-1}.
	\end{equation}

	For Algorithm 1-2 with LASSO-type penalty ($ q_i = 1$), we get 
	\begin{equation}\label{eq:lambda-max}
		\lambda_{\max} = \| W_n^{-1} \nabla \mathfrak L_n(0) \|_\infty,
	\end{equation}
where $ W_n $ is the diagonal matrix containing the weights $ \mathbf w_n=({\bf w}_n^1,...,{\bf w}_n^m)^\top $ on the main diagonal.
\end{proposition}

\begin{proof}
	
	In the LASSO-type case it is possible to directly manipulate subdifferentials.	In fact, in order for $ \theta = 0 $ to be a solution (dropping the subscript $ n $), it is required that 
	\[
	0 \in \nabla_{ij}\mathfrak L(0) + \lambda w_j^i s_{ij} \Leftrightarrow 
	\frac 1{w_j^i}\nabla_{ij}\mathfrak L(0)\in [-\lambda, \lambda ] \quad  \forall i,j.
	\]
	
		For $ q_i<1 $ we can proceed as follows. Note that the threshold value in \eqref{eq:tq-thresh-values}
	can be rewritten as
	\begin{equation}
		\quad t_{q, \lambda} = \lambda^{\frac{1}{2-q}} c_q, \qquad c_q  = [2(1 -q)]^{\frac{1}{2-q}} \left( 1 + \frac{q}{2(1-q)}\right)
	\end{equation}
	In order for a solution obtained by updates of the form \eqref{eq:prox-q-s2} to remain null it suffices that
	\begin{align}\label{eq:tq-0}
		\mathbf {\overline T}^{\mathbf q}_{\lambda s \mathbf w}(0) = 0
		&\Leftrightarrow |s \nabla_{ij} \mathfrak L(0) | \leq c_{q_i} (w^i_{j} s \lambda) ^{\frac 1 {2-q_i}}, \notag
		\\ & \Leftrightarrow \lambda \geq \left(c_{q_i}^{-1} | \nabla_{ij} \mathfrak L(0) |\right)^{2-q_i} \frac{s^{1-q_i}}{w^i_{j}},
	\end{align}
	$\forall \, 1 \leq i \leq m, \,\, 1 \leq j \leq p_i.$
	Since $ s < 1/L(\mathfrak L)$, inequalities \eqref{eq:tq-0} are satisfied by \eqref{eq:lambda-max-q}.
	
\end{proof}	
	
\begin{remark}
		For block-wise type Algoritm \ref{alg:pathwise3}, one can replace the global Lipschitz constant with the minimum of the partial Lipschitz constants. Note that \eqref{eq:lambda-max} can be seen as a special case of \eqref{eq:lambda-max-q} 
	for the quadratic loss case, in the limit as $ q_i\to 1^-  $. 
\end{remark}

\begin{lemma}\label{lem:lambda-max}
	Under the assumptions 
	\begin{equation}\label{eq:cond1}
	D_n:=A_n G_nA_n\stackrel{p}{\longrightarrow} D,
	\end{equation}
	where $D$ is a $\mathfrak p\times \mathfrak p$ positive definite symmetric random matrix
	\begin{equation}\label{eq:cond2}
	a^i_n r^i_n = O_p(1),
	\end{equation}
with $ a^i_n = \max\{w^i_{j}, j \leq p^0_i\} ,i=1,...,m,$
	for the bridge-type estimator $\hat \theta_n,$ $ \lambda_{\max}^{(n)} \to \infty $ in probability, as the sample size $ n\to \infty $; i.e. $P( \lambda_{\max}^{(n)}>M)\to 1,$ for any $M>0.$
\end{lemma}

We include here the superscript $ n $ to highlight the dependence of the maximal $ \lambda $ on the sample size.
\begin{proof}
In this case 
\[
\lambda_{\max}^{(n)} = \max_{\substack{1 \leq i \leq m \\ 1 \le j \leq p_i}} c_{q_i}^{q_i-2} \frac{|G_{ij} \tilde \theta |^{2-q_i}}{w^i_{j}} \|G\|^{q_i-1}
\]	
and note that $ G, \tilde \theta , \mathbf w$ depend on $ n $.
Let 
$
l_n = \min_{i} r^i_n. 
$
First note that, for every $ n $
\[
\|G\| = \max_{\|v\|=1} v^\top G v = \max_{\|v\|=1} (A_n^{-1}v)^\top D_n (A_n^{-1}v) \leq l_n^{-2} \|D_n\|
\]
and $ \|D_n\| = O_p(1) $ for the condition \eqref{eq:cond1}.

Denote with $ \ell = \ell(h,k) $ the index in $ \{1, \ldots , \mathfrak{p}\} $ corresponding to group $ h $, position $ k $. 
Also denote with $ \nu $ the index (depending on $ n $) of the smallest rate, i.e. $ r_n^\nu = l_n $. 

So, for $ j \leq p_0^i $, for any $ i $ and $ n $
\begin{align}\label{eq:lambda-conv1}
	&\frac{1}{w^i_j} |G_{ij} \theta^0|^{2 - q_i} \|G\|^{q_i -1} =
	\frac{1}{w^i_j} \left| \sum_{h=1}^m \sum_{k=1}^{p_h} (G_{ij})_{\ell(h,k)} \, \theta^0_{\ell(h,k)} \right|^{2-q_i} \|G\|^{q_i -1} 
	\\&=
	\frac{1}{w^i_j (r^i_n)^{2 - q_i}} \left| \sum_{h=1}^m \sum_{k=1}^{p_h} (D^n_{i,j})_\ell  \frac{ \theta^0_{\ell}}{r^{h}_n} \right|^{2-q_i} \|G\|^{q_i -1} 
	\notag \\
	&=
	\frac{1}{w^i_j (r^i_n)^{2 - q_i}}
	\left| 
	 \sum_{k'=1}^{p_\nu} (D^n_{i,j})_{\ell(\nu, k')}\,  \theta^0_{\ell(\nu, k')} +
	\sum_{h=1, h\neq \nu}^m \sum_{k=1}^{p_h} (D^n_{i,j})_{\ell(h, k)} \, \theta^0_{\ell(h, k)}\, \frac{l_n }{r^{h}_n} \right|^{2-q_i}
	\frac{ \|G\|^{q_i -1}}{l_n^{2-q_i}} 
	\notag \\ &\geq 
	 \frac{1}{a^i_n (r^i_n)^{2 - q_i}}
	 	\left| 
	 \sum_{k'=1}^{p_\nu} (D^n_{i,j})_{\ell(\nu, k')}\,  \theta^0_{\ell(\nu, k')} +
	\sum_{h=1, h\neq \nu}^m \sum_{k=1}^{p_h} (D^n_{i,j})_{\ell(h, k)} \, \theta^0_{\ell(h, k)}\, \frac{l_n }{r^{h}_n} \right|^{2-q_i}
	 l_n^{-q_i} \|D_n\|^{q_i -1}	 
	 \notag 
	 \\
	 &=: U_n l_n^{-q_i} (r^i_n)^{q_i -1} \to \infty,\quad n\to\infty,
	 \notag 
\end{align}
where $ a^i_n = \max\{w^i_{j}, j \leq p^0_i\} $. By assumptions \eqref{eq:cond1} and \eqref{eq:cond2}, the terms in the absolute value are $ O_p(1) $ as well as $ \|D_n\| $, since each $ (D^n_{i,j})_{\ell} $ converges in probability 
and $ 0 \leq l_n \leq  r^h_n $ (implying that the ratios $ l_n /r^h_n $ are bounded). 
Overall $ U_n = O_p(1) $ and, since $ r^i_n , l_n \to 0$, 
\eqref{eq:lambda-conv1} is unbounded.

Therefore
\begin{align}
	\lambda_{\max}^{(n)} = \max_{\substack{1 \leq i \leq m \\ 1 \le j \leq p_i}} c_{q_i}^{q_i-2} \frac{|G_{ij} \tilde \theta |^{2-q_i}}{w^i_{j}} \|G\|^{q_i-1} 
	\geq  \max_{\substack{1 \leq i \leq m \\ 1 \le j \leq p^0_i}} c_{q_i}^{q_i-2} \frac{|G_{ij} \tilde \theta |^{2-q_i}}{w^i_{j}} \|G\|^{q_i-1}
\end{align}
where the constants $ c_{q_i} $ do not depend on $ n $. For each term on the RHS, by convexity of $ x \mapsto x^{2-q_i} $
we can write
\begin{align}
	&\frac{1}{w^i_j} |G_{ij} \tilde \theta|^{2 - q_i} \|G\|^{q_i -1} = 
	\frac{1}{w^i_j} \left[ |G_{ij} \tilde \theta|^{2 - q_i} - |G_{ij} \theta_0|^{2 - q_i} +|G_{ij} \theta_0|^{2 - q_i} \right] \|G\|^{q_i -1}
	\\ &\geq 
	\frac{1}{w^i_j} \left[ (|G_{ij} \tilde \theta| - |G_{ij} \theta_0|) (2 - q_i) |G_{ij} \theta_0|^{1-q_i}
	 +|G_{ij} \theta_0|^{2 - q_i} \right] \|G\|^{q_i -1}
	 \notag \\ & \geq 
	 \frac{1}{w^i_j} \left[ -|G_{ij} (\tilde \theta -\theta_0)| (2 - q_i) |G_{ij} \theta_0|^{1-q_i}
	 +|G_{ij} \theta_0|^{2 - q_i} \right] \|G\|^{q_i -1}
	 \notag \\ &=
	 \frac{1}{w^i_j} \left[ -|r^i_n G_{ij} A_n A_n^{-1}(\tilde \theta -\theta_0)| (2 - q_i) \frac{1}{r^i_n|G_{ij} \theta_0|}
	 + 1 \right] |G_{ij} \theta_0|^{2 - q_i} \|G\|^{q_i -1} \to \infty 
	 \notag 
\end{align}
where the conclusion follows from \eqref{eq:lambda-conv1} and the fact that 
$ A_n^{-1}(\tilde \theta -\theta_0) = O_p(1) $, $ r^i_n G_{ij} A_n = O_p(1)$ element-wise, while 
\[
r^i_n|G_{ij} \theta_0| = \sum_{h,k} (D^n_{i,j})_\ell  \frac{ \theta^0_{\ell}}{r^{h}_n} \to \infty
\] 
making the terms into square brackets an $ O_p(1) $. The proof in the lasso can be performed by adapting the steps above.
\end{proof}

The fact that the solution path is identically null from some point on seems in contradiction with the consistency property of the estimators. The following theorem addresses the issue of path consistency. 

\begin{theorem}[Path consistency] 
	Under the assumptions \eqref{eq:cond3}, \eqref{eq:cond1} and \eqref{eq:cond2}, for the bridge-type estimator $\hat\theta_n$,
	as $ n\to \infty $, 
	the path estimates $ \{\hat \theta_n(\lambda): \lambda > 0 \} $
	produced either by
	\begin{enumerate}[(i)]
		\item any algorithm in the case with $ q_i = 1 \, \forall i; $ (LASSO-type)
		
		\item [] or
		
		\item the block-wise Algorithm \ref{alg:pathwise3}, 
		for $ 0 < q_i <1 $, 
		if the starting point is quite close to $\hat \theta_n$ (see Theorem 2.12 in \cite{Attouch2013ConvergenceOD});
	\end{enumerate}

	are pointwise consistent with respect to $\lambda$, i.e. 
	\[
	A_n^{-1}(\hat \theta_n(\lambda)- \theta_0) 1_{\{\lambda_{\max}^{(n)} > \lambda\}}=O_p(1), \quad n\to\infty, \forall \lambda > 0.
	\]
	Under additional technical  assumptions as in \cite{di-regularized}, the updates at point (i) and (ii) satisfy the oracle properties of variable selection and asymptotic normality.
\end{theorem}
\begin{proof}
	(i) In the LASSO case the stationary point coincides with the global minimum. For any fixed $\lambda >0$, by Lemma \ref{lem:lambda-max} and recalling that $\hat\theta_n$ satisfies the selection consistency property, it is not hard ti prove that for any $\varepsilon>0$
	$$P(A_{n,\varepsilon}):=P(\{ \lambda_{\max}^{(n)} > \lambda\}\cap \{|\hat \theta_n(\lambda)-\theta_0|\leq \varepsilon\})\longrightarrow 1,$$
	implying that the path is not identically null at such $ \lambda $. Then the consistency follows.
	
	(ii) In the bridge case the argument is analogous, where (local)
	convergence to a global minimum is guaranteed under the same conditions discussed for PALM Algorithm.

\end{proof}

\begin{remark}\label{rem:bridge-CD}
	If we consider each coordinate as a single block in a PALM-type algorithm, update \eqref{eq:udblsa}
	takes the form
	\begin{equation} \label{eq:update3}
		\hat \theta^{i,t}_j = \mathbf {T}^{q_i}_{\lambda s^i_j  w^i_j} \left(\hat \theta^{i, t-1}_j  - s^i_j G_{ij}\left((\hat \theta^{1,t},...,\hat \theta^{i,t}_{j-1},\hat \theta^{i,t-1}_j,...,\hat \theta^{m,t-1} )^\top- \tilde \theta \right)\right)
	\end{equation}
valid for step-sizes $ s^i_j < \alpha / g_{\ell \ell} $, $ g_{\ell \ell} $ being the Lipschitz constant of the one-dimensional block. We refer to this algorithm as the bridge version of the coordinate descent algorithm.
\end{remark}

\begin{figure}
	\begin{subfigure}{0.7\textwidth}
		\includegraphics[width=\textwidth]{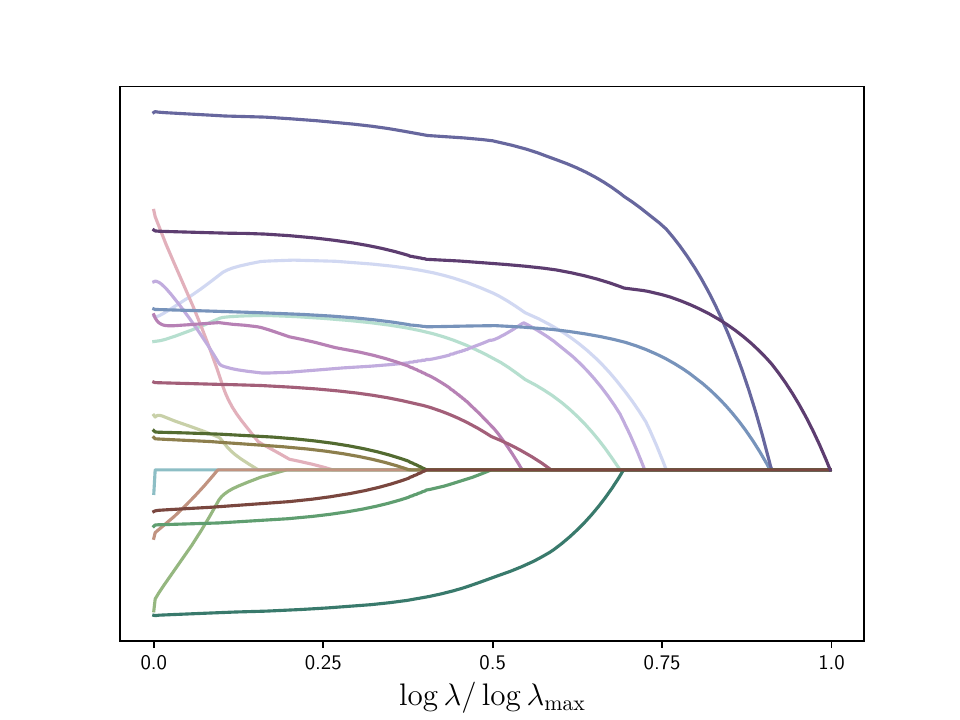}
		\caption{Sample path of the estimates for the LASSO estimator}
	\end{subfigure}
	\\
	\begin{subfigure}{0.7\textwidth}
		\includegraphics[width=\textwidth]{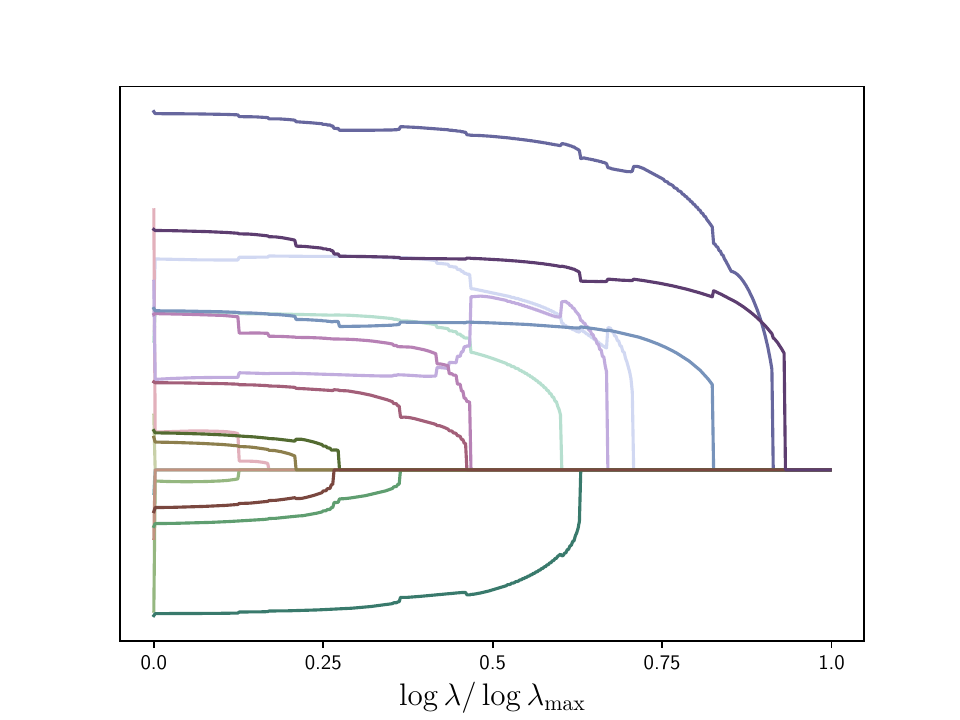}
		\caption{Sample path of the estimates of Algorithm 1}
	\end{subfigure}
	\caption{Comparison between LASSO and Bridge path. }
	\label{fig:paths}
\end{figure}

We close this section with a comparison between LASSO and Bridge-type solution paths. 
\autoref{fig:paths} shows typical sample paths outcome of Algorithm \ref{alg:pathwise2bis} and those arising in the LASSO case. The effect of the discontinuity of 
the $ T^q $ operators on the Bridge paths is apparent. The following sections discuss detailed comparisons of the techniques 
in terms of estimation performance.

\section{Applications to Generalized Linear Models}\label{sec:appglm}

\subsection{Penalized GLMs}
Estimator \eqref{eq:bridge-est2} can be applied to penalize likelihood functions arising from generalized linear models (GLMs). 
For instance, it is well known that for an exponential family of distributions with natural parameter $ \theta \in \Theta \subset \mathbb R^p$
\begin{equation}
	f(x; \theta) =  \exp\left\{\sum_{j=1}^{p} \theta_j T_j(x) - \psi(\theta) \right \} h(x)
\end{equation}
the natural parameter space $ \Theta $ is convex (\cite{lehmann2005testing}, Lemma 2.7.1) and, as an immediate consequence, the density is log-concave.
This means that the theoretical results regarding estimator \eqref{eq:bridge-est2} hold,  considering 
as the loss function $ \mathfrak L_n $ the negative log-likelihood 
\begin{equation}\label{eq:pen-lik}
	H_n(\theta; {q}):= \sum_{i=1}^{n} \ell_i(\theta)  + \lambda \|\theta\|_{q, \mathbf{w}_{n}}^{q},
\end{equation}
where $ \ell_i(\theta) = - \log f(x_i; \theta) $.
Note that in this case, we consider only one parameter group since there is a unique rate of convergence.

We would like to stress the fact that the problem of finding estimators based on \eqref{eq:pen-lik} has
been considered several times in the statistical literature (e.g. \cite{bridge-one-step}, \cite{huang2008bridge} and \cite{fu2000bridge} for regression problems). Nevertheless in all cases due to difficulties caused by the non-convexity of the problem,
the solution is computed in terms of some convex approximation. Our work shows that by exploiting non-convex analysis results the exact solution of the bridge estimator can be computed. Moreover the solution retains
the particular features of the estimator, which are variable selection combined with non-continuous estimation
path and penalization vanishing for large values of lambda. 

In particular in \cite{bridge-one-step} the one-step LLA estimator is introduced, derived as
\begin{equation}\label{eq:pen-lik-appr}
	\arg \min \left\{  \sum_{i=1}^{n} \ell_i(\theta)  + \sum_{j=1}^{p} p'_\lambda ( |\tilde \theta_j|) |\theta_j|\right\}
\end{equation}
where $ \tilde \theta $ denotes an initial (unpenalized) estimator for $ \theta $ and $ p'_\lambda $ 
is the derivative of a penalty function $ p_\lambda $. In the context of bridge estimation 
the penalty term in \eqref{eq:pen-lik-appr} becomes
\begin{equation}\label{eq:weigts-LLA}
	p'_\lambda ( |\tilde \theta_j|) |\theta_j| \propto \frac{1}{|\tilde \theta_j|^{1 - q}} |\theta_j|
\end{equation}
meaning that the proposed estimator is in fact equivalent to a weighted lasso estimator (e.g. compare with the choice of the weights \eqref{eq:lasso-weights-sde}). Note also that in \cite{bridge-one-step} a quadratic approximation of the log-likelihood is further considered, making the LLA one step estimator a particular 
case of the bridge-type estimator with weights given by \eqref{eq:weigts-LLA}.
In the case of regression problems with \emph{othonormal} designs, the "linear" approximation of the bridge penalty is capable of retaining the property of 
the vanishing penalization for large coefficients  (for an illustration see Figure 2 in \cite{bridge-one-step}). Specifically in this case, it is well known that the lasso solution is the soft-thresholding operator applied to the least squares estimator, incorporating the weights,  $ \hat \theta = S_{q\lambda/|\tilde \theta_j|^{1 - q}}(|\tilde \theta_j|) $. Note that the weights and the input are inversely related, causing the "vanishing penalization" effect.

However, for general quadratic approximation problems, it is not guaranteed that the penalization will vanish for large inputs, since the iterative updates of the form \eqref{eq:prox-q-ss}, for example, are based on weights which are now "fixed" with respect to the input of the operator.

\subsection{Simulation study}

Consider the linear regression model
\[
y = \mathbf X \theta_0 + \epsilon
\]

where $ y_i \in \mathbb R^n $, $ \mathbf X \in \mathbb R ^{n \times p}, \theta_0 \in \mathbb R^p $ and
$ \epsilon_i \sim  N(0, \sigma^2)$ are $ i.i.d.$ random variables.

The goal is to compare the predictive performance of the bridge estimator
\[
\hat \theta^q = \arg\min_{\theta}\left\{||y-{\bf X}\theta||_2^2+\lambda||\theta||_q^q\right\},
\]
its one step LLA approximation as in \eqref{eq:weigts-LLA}, denoted with $ \hat \theta^{q}_L
,$ and the 
classic lasso estimator $ \hat \theta^{1} $.

We consider a high-dimensional setting with $ p=500$.  
Roughly 2/3 of the components of $ \theta_0 $ are equal to zero (specifically 346), the others are positive or negative numbers in the range [-10, 10]. 
Following the same scheme as in \cite{zou2005regularization}, columns $ \mathbf{x}_i $ and $ \mathbf{x}_j $ of $ \mathbf X $ have correlation $ 0.5^{|i-j|} $. We set $ \sigma = 10 $.

The models are fitted by applying Algorithm \ref{alg:pathwise2bis} using $ n=1000 $ training points. The optimal $ \lambda $ value is chosen by cross-validation and the prediction error is measured on $ 1000 $ test points. 
LLA and LASSO solution were computed using existing R libraries, \texttt{grpreg} and \texttt{glmnet} respectively.

The results are summarized in \autoref{tab:regr-mse}, showing that Bridge estimator can achieve a better
performance in this context.
In \autoref{fig:regr-mse} we show the test error as a function of the normalized penalization parameter. 
In particular we see that Bridge can achieve a lower test error for suitable values of the tuning parameter, 
at the cost of a more rapid increase in bias as the penalization parameter increases. 

\begin{table}
	\begin{tabular}{cccc}
		\hline 
		& Bridge & LLA & LASSO \\
		\hline 
		Test $MSE$ & 25.85 & 26.01 & 26.9 \\
		\hline 
	\end{tabular}
	\caption{Test error comparison for bridge, one step LLA and LASSO estimators.}
	\label{tab:regr-mse}
\end{table}

\begin{figure}
	\begin{subfigure}{0.45\textwidth}
		\centering
		\includegraphics[width=0.99\textwidth]{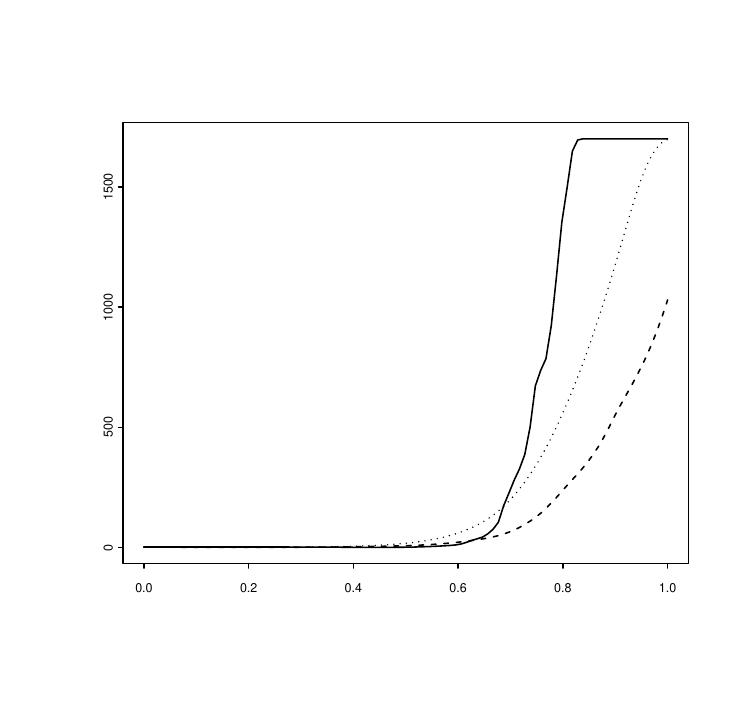}
		\caption{}
		\label{fig:regr-mse-test-full}
	\end{subfigure}
	\begin{subfigure}{0.45\textwidth}
		\centering
		\includegraphics[width=0.99\textwidth]{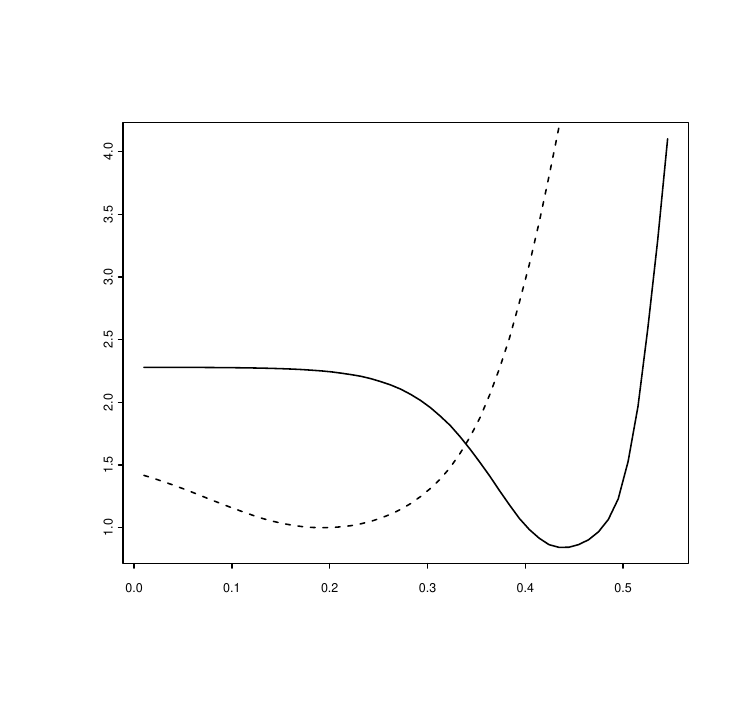}
		\caption{}
		\label{fig:regr-mse-test-red}
	\end{subfigure}%
	\caption{Comparison of test errors as a function of normalized penalization parameter $ \lambda / \lambda_{\max} $ for bridge (solid line), LLA (dashed), LASSO (dotted). 
 Figure B is a zoom of Figure A around the minimum of the curves.}
	\label{fig:regr-mse}
\end{figure}

\section{Application to Stochastic Differential Equations}\label{sec:appsde}

\subsection{Ergodic diffusions} 
Let  $(\Omega,\cor F,{\bf F}=(\cor F_t)_{t\geq 0},P)$ be a filtered complete probability space.
Let us consider a $d$-dimensional solution process $X:=(X_t)_{t\geq 0}$ to  the following stochastic differential equation (SDE)
\begin{equation}\label{sde}
	\de X_t=b(X_{t},\alpha)\de t+\sigma(X_{t},\beta)\de W_t,\quad X_0=x_0,
\end{equation}
where $x_0$ is a deterministic initial point, $b:\mathbb{R}^d\times  \Theta_{\alpha}\to \mathbb{R}^d$ and $\sigma:\mathbb{R}^d \times \Theta_\beta\to \mathbb{R}^d\otimes \re^r$ are Borel known functions (up to unknown parameter vectors $\alpha$ and $\beta$) and  $(W_t)_{t\geq 0}$ is a $r$-dimensional standard $\mathcal F_t$-Brownian motion. 
Let $\theta:=(\alpha,\beta)\in\Theta:=\Theta_\alpha\times \Theta_\beta \subset \mathbb R^{p_1 + p_2}$ and denote by $\theta_0:=(\alpha_0,\beta_0)$ the true value of $\theta$. 
As discussed in the Introduction, assume that $\theta_0$ has a sparse representation.

The process $X$ is sampled at $n+1$ equidistant discrete times $t_i^n$, such that $t_i^n-t_{i-1}^n=\Delta_n$ for $i=1,...,n,$ (with $t_0^n=0$). Therefore the data are observations ${\bf X}_n:=(X_{t_i^n})_{0\leq i\leq n}.$ 
We consider a \emph{high-frequency} sampling scheme, i.e with the following asymptotics as $ n\to\infty $: 
 $n\Delta_n\longrightarrow \infty$, $\Delta_n\longrightarrow 0$ and $n\Delta_n^p\longrightarrow 0$ for some $ p \geq 2 $, and there exists $\epsilon \in(0,(p-1)/p)$ such that $n^\epsilon\leq n\Delta_n$ for large $n$.
A widely used loss function in this context, for $ p=2 $, is  the negative quasi-log-likelihood function
 \begin{align}\label{qlik}
	\ell_n({\bf X}_n,\theta)
	&:=\frac12\sum_{i=1}^n\Bigg\{\log\text{det}(\Sigma( X_{t_{i-1}^n},\beta))
	\notag\\
	&\quad+\frac{1}{\Delta_n}(\Delta_i X-\Delta_n b( X_{t_{i-1}^n},\alpha))^\top \Sigma^{-1}( X_{t_{i-1}^n},\beta)(\Delta_i X-\Delta_n b( X_{t_{i-1}^n},\alpha))\Bigg\}.
\end{align}
with $\Sigma(x,\beta):=\sigma \sigma^\top (x,\beta)$ and $\Delta_i X:=X_{t_i^n}-X_{t_{i-1}^n} $, leading to the quasi-likelihood estimator
\[\tilde\theta_n = (\tilde\alpha_n,\tilde\beta_n) \in \arg\min_{\overline \Theta} \ell_n({\bf X}_n,\theta). \]

Under mild regularity conditions, for instance the functions $b$ and $\sigma$ are smooth, $\Sigma$ is invertible and $X$ is an ergodic diffusion, the quasi likelihood estimator has the following asymptotic properties (see e.g. \cite{kess} or \cite{yoshida2011polynomial}):

\begin{itemize}
	\item[(i)] $\tilde\alpha_n$ is $\sqrt{n\Delta_n}$-consistent while $\tilde\beta_n$ is $\sqrt{n}$-consistent; i.e. $(\sqrt{n\Delta_n}(\tilde\alpha_n-\alpha_0),\sqrt n(\tilde\beta_n-\beta_0))^\top=O_p(1);$
	\item[(ii)]   $\tilde\theta_n$ is asymptotically normal; i.e
	$$(\sqrt{n\Delta_n}(\tilde\alpha_n-\alpha_0),\sqrt n(\tilde\beta_n-\beta_0))^\top \stackrel{d}{\longrightarrow}N_{p_1+p_2}(0,\text{diag}((\Gamma^{11})^{-1},(\Gamma^{22})^{-1})),$$ 
\end{itemize}
	where $ \Gamma^{11} $ and $ \Gamma^{22} $ are the components of the Fisher information matrix of the diffusion.
%
%
%


From (i) and (ii) emerge that the estimator $\tilde\theta_n$ exhibits a \emph{mixed-rates} asymptotic regime with two different rates, $\sqrt{n\Delta_n}$ and $\sqrt n,$ for the two groups of parameters $\alpha$ and $\beta.$ 

The bridge-type estimator \eqref{eq:bridge-est2} can be implemented by setting $\mathfrak L_n(\theta):=\frac 12 (\theta-\tilde{\theta}_n)^\top  G_n(\theta-\tilde{\theta}_n),$ and $G_n = \ddot \ell_n({\bf X}_n,\theta)$ (the Hessian matrix). 
We consider adaptive weights of the form 
\begin{align}\label{eq:lasso-weights-sde}
	w_{n,j}^1 = \frac{w_{n,0}^1}{|\tilde \alpha_{n,j}|^{\delta_1}},
	\qquad 
	w_{n,k}^2 =  \frac{w_{n,0}^2}{|\tilde \beta_{n,k}|^{\delta_2}}
	\qquad j = 1, , \ldots, p_1, \,\,k = 1, \ldots, p_2,
\end{align}
where the exponents $ \delta_1 $ and $ \delta_2 $ are such that $ \delta_i > 1-q_i,1,2$ and
$ w_{n,0}^i $ have suitable asymptotics (constants will do for $ \delta_i > 2-q_i $).
In this setting the oracle properties of the regularized estimator \eqref{eq:bridge-est2} hold by setting $A_n= \text{diag}(1/(\sqrt{n\Delta_n}){\bf I}_{p_1},1/\sqrt n{\bf I}_{p_2})$. See \cite{di-regularized} for details.

It is worth mentioning that in literature appeared different estimators for ergodic diffusions satisfying the asymptotic properties (i) and (ii), for instance the quasi-Bayesian estimator or the hybrid multistep estimator (\cite{uchida2012adaptive}, \cite{uchida2014adaptive}, \cite{kamatani2015hybrid}).

\subsection{Simulation study.}

Consider a multivariate diffusion process $ X = (X^1, \ldots X^d)^\top= (X_t)_{ t\geq 0} \in \mathbb{R}^d $ driven by the SDE
\begin{equation}\label{eq:sim-sde-vec}
	\mathrm d X_t  = A X_t \mathrm d t + 
	B \, \mathrm d W_t
\end{equation}		
where $ A = (\alpha_{i,j}: \, i,j = 1, \ldots d)$ and $ B = (\beta_{i,j}:\, i,j = 1, \ldots d)  $
are parameter matrices assumed to be positive definite. In our case $ d = 4 $ and the true parameter matrices 
are 
\begin{equation}
	A = \begin{bmatrix}
		4 & -1.8&  0 & 0 \\
		0 & 4 & -1.8 & 0 \\
		0 & 0 & 4 & -1.8 \\
		0 & 0 & 0 & 4
	\end{bmatrix}
	\qquad 
	B = \begin{bmatrix}
		4 & 0&  0 & 0 \\
		0 & 4 & 0 & 0 \\
		0 & 0 & 4 & 0 \\
		0 & 0 & 0 & 4
	\end{bmatrix}
\end{equation}

\begin{figure}
	\includegraphics[width=\textwidth]{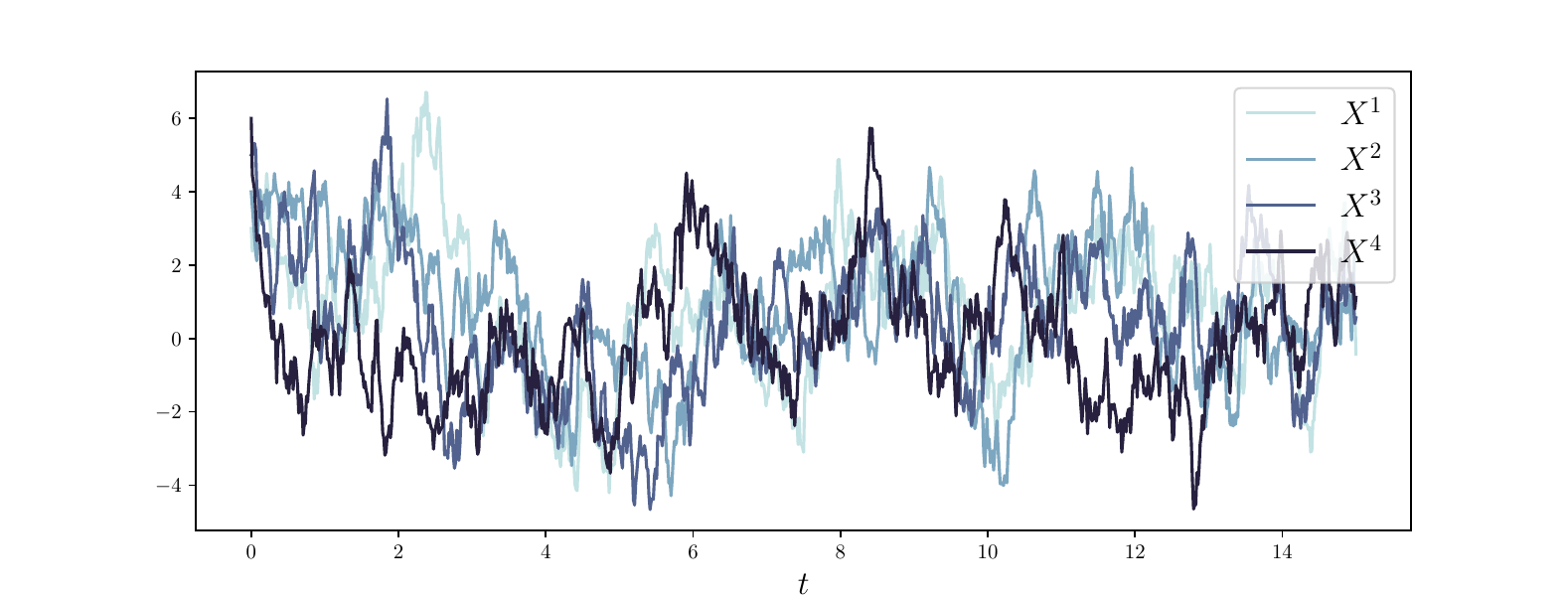}
	\caption{A simulated sample path of the solution to \eqref{eq:sim-sde-vec}}
\end{figure}

In this case the process components are not independent because of the dependence in the drift term. However
the correlation between the components can be ascribed to a specific ``causality" structure in the sense of Granger (\cite{di-regularized} and references therein for details): e.g. it is clear that $X^4$ is not caused by the other components whereas it causes $X^3$.

Therefore the goal of this study is to show how the estimator \eqref{eq:bridge-est2} with $\mathfrak L_n(\theta):=\frac 12 (\theta-\tilde{\theta}_n)^\top  G_n(\theta-\tilde{\theta}_n)$ and $G_n = \ddot \ell_n({\bf X}_n,\theta),$
is able to highlight the true dependency structure between the observed random signals
by starting with ``full" parameter matrices and by setting to zero the parameters
corresponding to non-existing relations.

In order for the model to be identifiable the matrix $B$ is restricted to be upper triangular, 
so the parameters $\beta_{i,j}, i<j$ are set to zero and are not to be estimated.
Moreover to enforce positive definiteness the diagonal elements $A$ and $B$ are restricted to be positive\footnote{For $A$ this is only a necessary condition. A further check is made during the optimization process to return a high objective function value if the condition is not met, and the optimization is repeated with a new random starting point until an admissible solution is found.}.
So the parameter space $\Theta \subset \mathbb R^{26}$ is the cartesian product of intervals either of the form $[-50, 50]$ or 
$[0, 50]$ for diagonal elements.

Discretely sampled data from model \eqref{eq:sim-sde-vec} are simulated from the true model. 
We consider the cases $n=1000$ and $n=10000$, 
corresponding to $\Delta_n = 0.015, T=15$, $\Delta_n = 0.003, T=30$ respectively, 
where $n$ denotes the length of the sampled series, $\Delta_n$ the time interval between subsequent observations
and $T$ is observation horizon.

The initial estimator $\tilde\theta_n:=(\tilde\alpha_n,\tilde\beta_n)$ is derived by optimizing \eqref{qlik}.
The penalized estimator \eqref{eq:bridge-est2} is obtained by setting
\begin{equation}
	w^1_{n,j} = \frac{1}{|\tilde \alpha_{n,j}|^{4}},
	\qquad 
	w^2_{n,k} =  \frac{1}{|\tilde \beta_{n,k}|^{4}}
	\qquad j = 1, , \ldots, p_1, \,\,k = 1, \ldots, p_2.
\end{equation}
and for the Bridge estimator we set $q_1 = q_2 = 1/2$.

We estimate the overall relative mean square error $ MSE_{\theta_0} = \mathbb E \|\hat \theta - \theta_0\|^2 / \|\theta_0\|^2 $ and the model selection probability $ P_0 $, i.e. the probability of estimating exactly as zero all of the zero parameters and only those. 
Since the oracle properties are asymptotic results for finite samples a small margin of error in model selection can be expected. Thus we report the estimates for {\it approximate} $ P_0 $, given by the probability of correctly identifying the true relations with at most one zero parameter not estimated as zero or vice versa. 
The results of the simulation study are shown in \autoref{tab:res}. The column corresponding to $ \tilde \lambda_{opt} $ shows the best performance achieved along the path. In all cases $  \tilde \lambda_{opt} \approx 0.5 $. We can draw the following conclusions:

\begin{itemize}
	\item Adaptive penalized techniques can improve convergence of the estimates in a mean square error sense for mild penalization 
	values with respect to the initial QML estimates. Larger penalization values, closer to $ \lambda_{\max} $, will cause the performance to deteriorate instead, since the estimates are heavily shrunk to zero.
	
	\item 
	Penalized techniques with $n=10000$ reach a high probability of correctly identify the true model. The best accuracy achieved is around 59 \% for both methods. If we allow for a small margin of error (at most one parameter not set to zero or vice versa),
	we obtain an approximate model selection probability higher than  96\%. Note that for smaller penalization 
	Bridge selection probability is much higher than the lasso.
	 Bridge selection shows a generally higher selection probability for almost all penalization values, as shown in \autoref{fig:sel-prob}, while at 
	the peak value where they tend to coincide, also for $ n=1000 $. This suggests that Bridge estimator is able to speed-up the convergence of the estimator for smaller sample sizes and is more robust to the choice of the tuning parameter. This is 
	useful in real-world applications where the tuning parameter has to be chosen using some validation technique: a sub-optimal choice of $ \lambda $ will impact less bridge estimates.
	
	\item \textit{Algorithms efficiency}. 
	\autoref{fig:comparison} shows a comparison of the average number of iteration required by Algorithms 1 and 3, 
	for both bridge and LASSO variants. The block variant is implemented by considering the drift and diffusion parameters as blocks, i.e. $ (\theta^1, \theta^2) = (\alpha, \beta)$. 
 We also compare these algorithms with the coordinate descent, which is based on univariate updates of the parameter vector (see \cite{tseng2001}), and is commonly used in statistical applications (\cite{hastiebook2015}, chapter 5).
 The coordinate descent for bridge is meant in the sense of block-wise algorithm with one block parameters, as in Remark \ref{rem:bridge-CD}. The fastest algorithm is the accelerated GD as could be expected. Interestingly enough
	the block-wise algorithm can achieve a similar performance, even if not accelerated. The results suggest that taking into account the block structure of the problem lead to an improvement in computation time by a factor of $ \tilde 3 $ with respect to coordinate-wise algorithms. 
\end{itemize}
\begin{table}
	\scriptsize 
	\renewcommand{\arraystretch}{1.2}
	\begin{tabular}{cc| ccccccc } 
		\hline 
		&& QMLE & \multicolumn{3}{c}{LASSO}  & \multicolumn{3}{c}{Bridge}  \\
		&& 	 & $ \tilde \lambda =0.25$ &  $ \tilde \lambda =\tilde \lambda_{opt}$  &  $ \tilde \lambda =0.7$  &  $ \tilde \lambda =0.25$   &  $ \tilde \lambda =\tilde \lambda_{opt}$   &  $ \tilde \lambda =0.7$ 	\\
		\hline 
		\multirow{2}{*}{$MSE_{\theta_0}$} &
		$n=1000$ & 0.072  &  0.058  &  0.052  &  0.085  &   0.058  &   0.051   &  0.074 \\
		&
		$n=10000$ & 0.035 &   0.023  &  0.029  &  0.093   &  0.023  &   0.027  &   0.075 \\
		\hline 
		\multirow{2}{*}{$P_0$} &
		$n=1000$ & - & 0.025 & 0.271    & 0.033    &  0.065 & 0.278    &     0.034    \\
		&
		$n=10000$& - & 0.168   &  0.590 &   0.019   &  0.323    & 0.598   &  0.021 \\
		\hline 
		\multirow{2}{*}{$ appr. \, P_0$} &
		$n=1000$& - &   0.145 & 0.765    & 0.237    & 0.294  &   0.775  &    0.249    \\
		&
		$n=10000$ & - & 0.519   &  0.962 &   0.183  &  0.730   &  0.963    & 0.199  \\
		\hline 
	\end{tabular}
	\caption{Numerical results.}
	\label{tab:res}
\end{table}

\begin{figure}
	\includegraphics[width=0.6\textwidth]{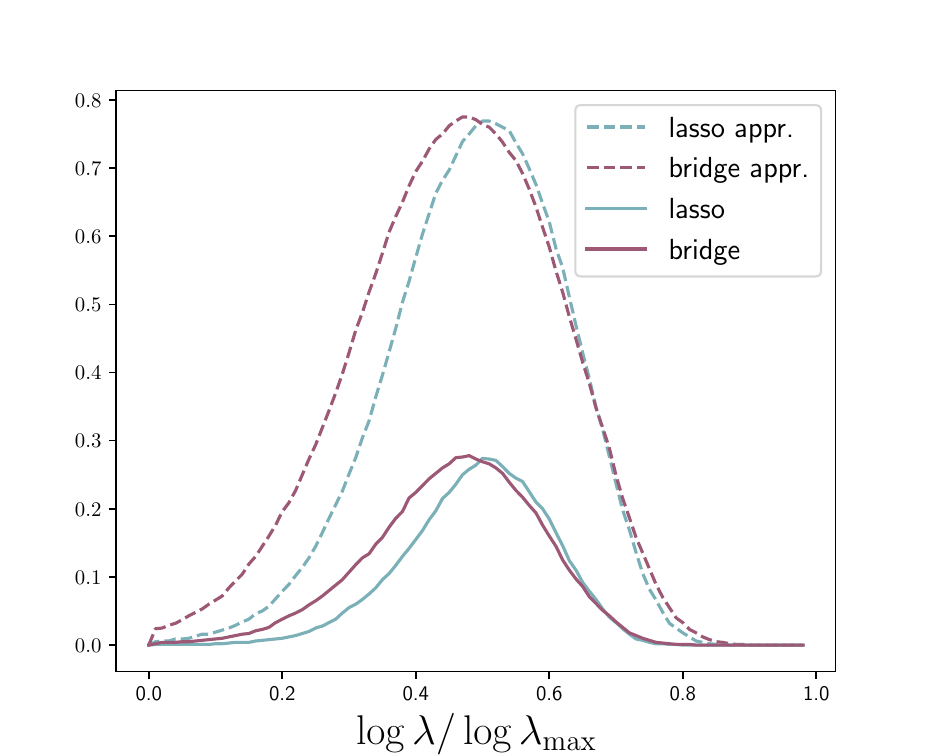}
	\caption{Model selection proportions for $n=10^3$ as a function of the
		normalized penalization parameter.}
	\label{fig:sel-prob}
\end{figure}

\begin{figure}

	\includegraphics[width=.7\textwidth]{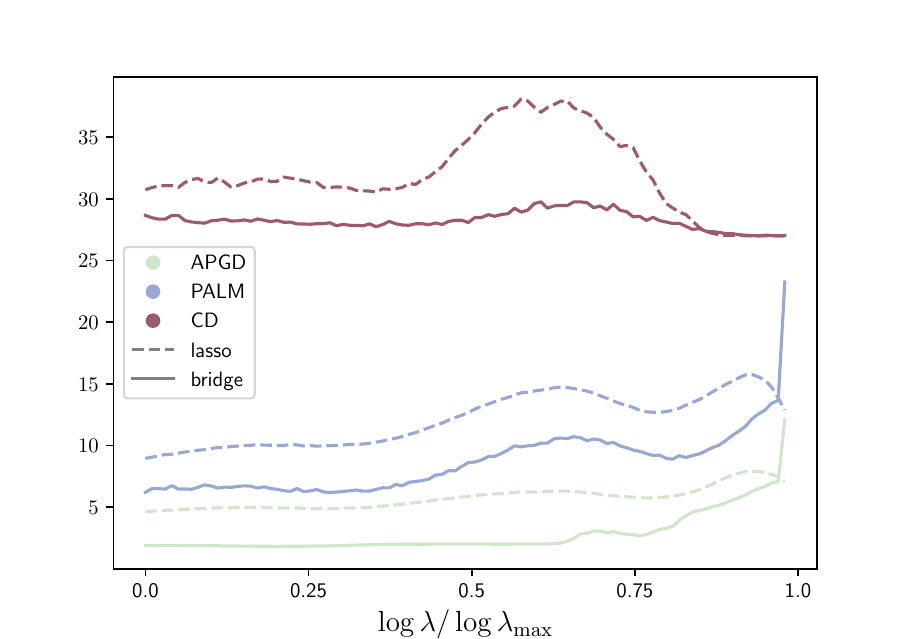}
	\caption{Comparison between algorithms \emph{APGD} (Alg.1), \emph{PALM} (Alg. 3),  \emph{CD} (Coordinate Descent) for lasso (dashed) and bridge (solid line): number of iterations as a function of the normalized penalization parameter. In CD in the bridge case is meant to be (Alg.3) with one-parameter blocks ($ n = 10^4 $).}
	\label{fig:comparison}
\end{figure}

\section*{Appendix}
Let us recall few definitions concerning subdifferential calculus and variational analysis for nonconvex and nonsmooth functions (see, e.g., \cite{RockWets09}).
A function $f: \mathbb R^n \to (-\infty, \infty]$ is said to be proper if $\text{dom}\, f\neq \emptyset$, where $\text{dom}\, f = \{ x\in \mathbb R^n : f(x) <+\infty\}$ . Furthermore, $h$ is lower semicontinuous at point $x_0$ if $\lim\inf_{x\to x_0} f(x) \geq f(x_0)$.  Furthermore  $f$ is coercive if it is bounded from below and $f(x)\to \infty$ when $||x|| \to \infty$, where $||\cdot||$ stands for the euclidean norm.

A function $f:\mathbb R^n\to \mathbb R$ is of type $C^{1,1}$; i.e. continuously differentiable  with Lipschitz continuous gradient, if
$$||\nabla f(x)-\nabla f(y)||\leq L(f)||x-y||,\quad \forall x,y\in \mathbb R^n,$$
where  $L(f)>0$ is the global Lipschitz constant of $\nabla f.$

Now, we introduce the subdifferential for a function. This mathematical tool plays a crucial role in our analysis.
\begin{definition}[Subdifferential]
Let $f: \mathbb R^n \to (-\infty, \infty]$ be a proper and lower semicontinuous function.

(i) The Frechét subdifferential of $f$ at $x\in$dom\,$f$, written $\widehat \partial f(x),$ is defined as follows
$$\widehat \partial f(x):=\left\{u\in \mathbb R^n:	 \lim_{y\neq x}\inf_{y	\to x}\frac{f(y)-f(x)-u^\top (y-x)}{||y-x||}\geq 0\right\}.$$
When $x\notin$dom\,$f$, we set $\widehat \partial f(x)=\emptyset.$

(ii) The limiting-subdifferential, or simply the subdifferential, of $f$  at $x\in\mathbb R^n,$ written $\partial f(x),$ is defined as follows
$$\partial f(x):=\{u\in \mathbb R^n: \exists \, x^k\to x,\, f(x^k)\to f(x),\, \text{and} \,\, \widehat \partial f(x^k)\ni u^k\to u\, \text{as}\, k\to 	\infty\}.$$

(iii) $x\in \mathbb R^n$ is called critical point if $0\in \partial f(x).$
\end{definition}

The well known Fermat’s rule remains unchanged; that is if $x\in \mathbb R^n$ is a local minimizer of $f$ then $0\in \partial f(x)$ (see Theorem 10.1 in  \cite{RockWets09}). 

If $f:\mathbb R^n \to(-\infty,\infty]$ is a proper and convex function, as particular case we obtain the definition of subdifferential for a convex function (see \cite{bertsekas2009convex} and Proposition 8.12 in  \cite{RockWets09}); i.e. for any $x\in$ dom $f$
$$\partial f(x)=\{u\in\mathbb R^n: f(y)\geq f(x)+ u^\top (y-x), \forall y\}=\widehat\partial f(x).$$
When $f$ is also differential at $x,$ $u\in\partial f(x)$ coincides with $\nabla f(x).$ 

Now, we recall the essentials issues about the  Kurdyka–Łojasiewicz (KL) property (see \cite{attouch2010} for more details) playing a crucial role in the nonconvex optimization theory. These property guarantees the global convergence of the whole sequence generated by a nonconvex algorithms to a critical point (see, e.g, \cite{attouch2010}, \cite{bolte2014proximal} and \cite{li2015accelerated}). 

Let $\Phi_\eta,	\eta\in(0,+\infty],$ be the class of concave and continuous functions $\varphi:[0,\eta)\to \mathbb R_+$ such that the following conditions fulfill: 1) $\varphi(0)=0;$ 2) $\varphi \in C^1((0,\eta))$ and $\varphi$ is continuous in 0; 3) $\varphi'(s)>0$ for all $s\in(0,\eta).$ 

\begin{definition}[Kurdyka–Łojasiewicz property; see  \cite{attouch2010}] 

Let  $f: \mathbb R^n \to (-\infty, \infty]$ be a proper lower semicontinuous function. Then $f$ is said to have  the  KL property if for any $\bar u\in$dom $\partial f:=\{u\in \mathbb R^n: \partial f(u)\neq \emptyset\}$ there exists $\eta\in(0,+\infty],$ a neighborhood $U$ of $\bar u$ and a function $\varphi\in\Phi_\eta$ such that for all $u\in U\cap \{u\in\mathbb R^n: f(\bar u)<f(u)<f(\bar u+\eta)\}$ the following inequality holds
$$\varphi'(f(u)-f(\bar u))\text{dist}(0,\partial f(u))\geq 1.$$
\end{definition}

The KL property is not so restrictive. Several functions appearing in the optimization problems, such as  semi-algebraic functions are KL functions (see \cite{attouch2010} and \cite{bolte2014proximal} for a formal definition of semi-algebraic function and its property).
  Typical semi-algebraic functions include real polynomial functions,  finite sums/product and compositions of semi-algebraic functions, $||x||_0,$ $||x||_p,$ with $p> 0$ and rational.

	\bibliographystyle{abbrv}
	\bibliography{biblio}

\begin{thebibliography}{10}

\bibitem{antoine2012efficient}
B.~Antoine and E.~Renault.
\newblock Efficient minimum distance estimation with multiple rates of
  convergence.
\newblock {\em Journal of Econometrics}, 170(2):350--367, 2012.

\bibitem{attouch2010}
H.~Attouch, J.~Bolte, P.~Redont, and A.~Soubeyran.
\newblock Proximal alternating minimization and projection methods for
  nonconvex problems: An approach based on the kurdyka-Łojasiewicz inequality.
\newblock {\em Mathematics of Operations Research}, 35(2):438--457, 2010.

\bibitem{Attouch2013ConvergenceOD}
H.~Attouch, J.~Bolte, and B.~F. Svaiter.
\newblock Convergence of descent methods for semi-algebraic and tame problems:
  proximal algorithms, forward–backward splitting, and regularized
  gauss–seidel methods.
\newblock {\em Mathematical Programming}, 137:91--129, 2013.

\bibitem{5173518}
A.~Beck and M.~Teboulle.
\newblock Fast gradient-based algorithms for constrained total variation image
  denoising and deblurring problems.
\newblock {\em IEEE Transactions on Image Processing}, 18(11):2419--2434, 2009.

\bibitem{fista}
A.~Beck and M.~Teboulle.
\newblock A fast iterative shrinkage-thresholding algorithm for linear inverse
  problems.
\newblock {\em SIAM Journal on Imaging Sciences}, 2(1):183--202, 2009.

\bibitem{bertsekas2009convex}
D.~Bertsekas.
\newblock {\em Convex Optimization Theory}.
\newblock Athena Scientific optimization and computation series. Athena
  Scientific, 2009.

\bibitem{bertsekas2015convex}
D.~Bertsekas.
\newblock {\em Convex Optimization Algorithms}.
\newblock Athena Scientific, 2015.

\bibitem{bolte2014proximal}
J.~Bolte, S.~Sabach, and M.~Teboulle.
\newblock Proximal alternating linearized minimization for nonconvex and
  nonsmooth problems.
\newblock {\em Mathematical Programming}, 146(1):459--494, 2014.

\bibitem{boyd_vandenberghe_2004}
S.~Boyd and L.~Vandenberghe.
\newblock {\em Convex Optimization}.
\newblock Cambridge University Press, 2004.

\bibitem{breheny2011coordinate}
P.~Breheny and J.~Huang.
\newblock Coordinate descent algorithms for nonconvex penalized regression,
  with applications to biological feature selection.
\newblock {\em The annals of applied statistics}, 5(1):232, 2011.

\bibitem{candes2008enhancing}
E.~J. Candes, M.~B. Wakin, and S.~P. Boyd.
\newblock Enhancing sparsity by reweighted $\ell_1$ minimization.
\newblock {\em Journal of Fourier analysis and applications}, 14(5):877--905,
  2008.

\bibitem{lqoptim}
F.~Chen, L.~Shen, and B.~Suter.
\newblock Computing the proximity operator of the $\ell_p$ norm with $0 < p <
  1$.
\newblock {\em IET Signal Processing}, 10(5):557--565, July 2016.
\newblock Publisher Copyright: {\textcopyright} The Institution of Engineering
  and Technology 2016.

\bibitem{ciolek2020dantzig}
G.~Cio{\l}ek, D.~Marushkevych, and M.~Podolskij.
\newblock On dantzig and lasso estimators of the drift in a high dimensional
  ornstein-uhlenbeck model.
\newblock {\em Electronic Journal of Statistics}, 14(2):4395--4420, 2020.

\bibitem{ciolek2022lasso}
G.~Ciolek, D.~Marushkevych, and M.~Podolskij.
\newblock On lasso estimator for the drift function in diffusion models.
\newblock {\em arXiv preprint arXiv:2209.05974}, 2022.

\bibitem{de2012adaptive}
A.~De~Gregorio and S.~M. Iacus.
\newblock Adaptive lasso-type estimation for multivariate diffusion processes.
\newblock {\em Econometric Theory}, 28(4):838--860, 2012.

\bibitem{di-regularized}
A.~De~Gregorio and F.~Iafrate.
\newblock Regularized bridge-type estimation with multiple penalties.
\newblock {\em Annals of the Institute of Statistical Mathematics}, pages
  1--31, 2020.

\bibitem{fanli2001}
J.~Fan and R.~Li.
\newblock Variable selection via nonconcave penalized likelihood and its oracle
  properties.
\newblock {\em Journal of the American Statistical Association},
  96(456):1348--1360, 2001.

\bibitem{frank1993statistical}
L.~E. Frank and J.~H. Friedman.
\newblock A statistical view of some chemometrics regression tools.
\newblock {\em Technometrics}, 35(2):109--135, 1993.

\bibitem{fu2000bridge}
W.~Fu and K.~Knight.
\newblock Asymptotics for lasso-type estimators.
\newblock {\em The Annals of statistics}, 28(5):1356--1378, 2000.

\bibitem{hastiebook2015}
T.~Hastie, R.~Tibshirani, and M.~Wainwright.
\newblock {\em Statistical Learning with Sparsity: The Lasso and
  Generalizations}.
\newblock Chapman \& Hall/CRC, 2015.

\bibitem{huang2008bridge}
J.~Huang, J.~L. Horowitz, and S.~Ma.
\newblock Asymptotic properties of bridge estimators in sparse high-dimensional
  regression models.
\newblock {\em The Annals of Statistics}, 36(2):587--613, 2008.

\bibitem{group-bridge}
J.~Huang, S.~Ma, H.~Xie, and C.-H. Zhang.
\newblock A group bridge approach for variable selection.
\newblock {\em Biometrika}, 96(2):339--355, 2009.

\bibitem{hunter2005variable}
D.~R. Hunter and R.~Li.
\newblock Variable selection using mm algorithms.
\newblock {\em Annals of statistics}, 33(4):1617, 2005.

\bibitem{kamatani2015hybrid}
K.~Kamatani and M.~Uchida.
\newblock Hybrid multi-step estimators for stochastic differential equations
  based on sampled data.
\newblock {\em Statistical Inference for Stochastic Processes}, 18(2):177--204,
  2015.

\bibitem{kess}
M.~Kessler.
\newblock Estimation of an ergodic diffusion from discrete observations.
\newblock {\em Scandinavian Journal of Statistics}, 24(2):211--229, 1997.

\bibitem{kinoshita2019penalized}
Y.~Kinoshita and N.~Yoshida.
\newblock Penalized quasi likelihood estimation for variable selection.
\newblock {\em arXiv preprint arXiv:1910.12871}, 2019.

\bibitem{koike2020graphical}
Y.~Koike.
\newblock De-biased graphical lasso for high-frequency data.
\newblock {\em Entropy}, 22(4), 2020.

\bibitem{lehmann2005testing}
E.~L. Lehmann, J.~P. Romano, and G.~Casella.
\newblock {\em Testing statistical hypotheses}, volume~3.
\newblock Springer, 2005.

\bibitem{li2015accelerated}
H.~Li and Z.~Lin.
\newblock Accelerated proximal gradient methods for nonconvex programming.
\newblock {\em Advances in neural information processing systems}, 28:379--387,
  2015.

\bibitem{lq-thresholding}
G.~Marjanovic and V.~Solo.
\newblock On $l_q$ optimization and matrix completion.
\newblock {\em IEEE Transactions on Signal Processing}, 60(11):5714--5724,
  2012.

\bibitem{masuda2017moment}
H.~Masuda and Y.~Shimizu.
\newblock Moment convergence in regularized estimation under multiple and
  mixed-rates asymptotics.
\newblock {\em Mathematical Methods of Statistics}, 26(2):81--110, 2017.

\bibitem{mazumder2011sparsenet}
R.~Mazumder, J.~H. Friedman, and T.~Hastie.
\newblock Sparsenet: Coordinate descent with nonconvex penalties.
\newblock {\em Journal of the American Statistical Association},
  106(495):1125--1138, 2011.

\bibitem{nesterov1983method}
Y.~E. Nesterov.
\newblock A method for unconstrained convex minimization problem with the rate
  of convergence $o(1/k^2)$.
\newblock {\em Soviet Mathematics Doklady}, 27(2):372--376, 1983.

\bibitem{ipalm}
T.~Pock and S.~Sabach.
\newblock Inertial proximal alternating linearized minimization (ipalm) for
  nonconvex and nonsmooth problems.
\newblock {\em SIAM Journal on Imaging Sciences}, 9(4):1756--1787, 2016.

\bibitem{rad}
P.~Radchenko.
\newblock Mixed-rates asymptotics.
\newblock {\em The Annals of Statistics}, 36(1):287--309, 2008.

\bibitem{RockWets09}
R.~Rockafellar and R.~J.-B. Wets.
\newblock {\em Variational Analysis}.
\newblock Springer Verlag, Heidelberg, Berlin, New York, 2009.

\bibitem{suz}
T.~Suzuki and N.~Yoshida.
\newblock Penalized least squares approximation methods and their applications
  to stochastic processes.
\newblock {\em Japanese Journal of Statistics and Data Science}, 3, 12 2020.

\bibitem{tib}
R.~Tibshirani.
\newblock Regression shrinkage and selection via the lasso.
\newblock {\em Journal of the Royal Statistical Society. Series B
  (Methodological)}, 58(1):267--288, 1996.

\bibitem{tseng2001}
P.~Tseng.
\newblock Convergence of a block coordinate descent method for
  nondifferentiable minimization.
\newblock {\em Journal of optimization theory and applications},
  109(3):475--494, 2001.

\bibitem{uchida2012adaptive}
M.~Uchida and N.~Yoshida.
\newblock Adaptive estimation of an ergodic diffusion process based on sampled
  data.
\newblock {\em Stochastic Processes and their Applications}, 122(8):2885--2924,
  2012.

\bibitem{uchida2014adaptive}
M.~Uchida and N.~Yoshida.
\newblock Adaptive bayes type estimators of ergodic diffusion processes from
  discrete observations.
\newblock {\em Statistical Inference for Stochastic Processes}, 17(2):181--219,
  2014.

\bibitem{wang1}
H.~Wang and C.~Leng.
\newblock Unified lasso estimation by least squares approximation.
\newblock {\em Journal of the American Statistical Association},
  102(479):1039--1048, 2007.

\bibitem{block-non-convex}
Y.~Xu and W.~Yin.
\newblock A globally convergent algorithm for nonconvex optimization based on
  block coordinate update.
\newblock {\em Journal of Scientific Computing}, 72(2):700--734, 2017.

\bibitem{l12-regularization}
Z.~Xu, X.~Chang, F.~Xu, and H.~Zhang.
\newblock $l_{1/2}$ regularization: A thresholding representation theory and a
  fast solver.
\newblock {\em IEEE Transactions on Neural Networks and Learning Systems},
  23(7):1013--1027, 2012.

\bibitem{yoshida2011polynomial}
N.~Yoshida.
\newblock Polynomial type large deviation inequalities and quasi-likelihood
  analysis for stochastic differential equations.
\newblock {\em Annals of the Institute of Statistical Mathematics},
  63(3):431--479, 2011.

\bibitem{zhang2010}
C.-H. Zhang.
\newblock {Nearly unbiased variable selection under minimax concave penalty}.
\newblock {\em The Annals of Statistics}, 38(2):894 -- 942, 2010.

\bibitem{zou2005regularization}
H.~Zou and T.~Hastie.
\newblock Regularization and variable selection via the elastic net.
\newblock {\em Journal of the royal statistical society: series B (statistical
  methodology)}, 67(2):301--320, 2005.

\bibitem{bridge-one-step}
H.~Zou and R.~Li.
\newblock {One-step sparse estimates in nonconcave penalized likelihood
  models}.
\newblock {\em The Annals of Statistics}, 36(4):1509 -- 1533, 2008.

\end{thebibliography}

\end{document}